%% file: paper-element-privacy.tex
\documentclass[11pt]{article}

\usepackage{macros}

\usepackage{mathtools}
\usepackage{color}
\usepackage{comment}
\usepackage{xspace}
\usepackage{tikz}
\usepackage{algorithm}
\usepackage[noend]{algpseudocode}
\usepackage{overpic}
\usepackage{nicefrac}

\definecolor{darkblue}{rgb}{0,0,.75}
\usepackage{graphicx}
\usepackage[colorlinks=true,linkcolor=darkblue,%
citecolor=darkblue,urlcolor=darkblue]{hyperref}
\usepackage[numbers]{natbib}

\usepackage{enumitem}
\usepackage[top=2.54cm,right=2.54cm,left=2.54cm,bottom=2.54cm]{geometry}

\newcommand{\distuser}{d_{\mathsf{user}}}
\newcommand{\distsample}{d_{\mathsf{sample}}}
\newcommand{\dham}{d_{\mathsf{Hamming}}}
\newcommand{\mechanism}{\mathsf{M}}
\newcommand{\diffp}{\varepsilon}
\newcommand{\renparam}{\alpha}
\newcommand{\drenyi}[2]{D_\renparam\left({#1} |\!| {#2}\right)}

\newcommand{\uniformdist}{\mathsf{Uni}}

\newcommand{\aprox}{\mathsf{aprox}^\loss}
\newcommand{\gradmap}{\mathsf{g}}
\newcommand{\lipconst}{\mathsf{L}}
\newcommand{\lipobj}{\lipconst_0}
\newcommand{\lipgrad}{\lipconst_1}
\newcommand{\liphess}{\lipconst_2}
\newcommand{\ball}{\mathbb{B}}
\newcommand{\opt}{^\star}
\newcommand{\normalvec}{v}
\newcommand{\ones}{\mathbf{1}}
\newcommand{\randind}{\mathsf{ri}}
\newcommand{\Zcov}{\Sigma_{\textup{z}}} 
\newcommand{\losscov}{\Sigma_{\loss}} 
\newcommand{\statrv}{X}
\newcommand{\statval}{x}
\newcommand{\statdomain}{\mc{X}}

\newcommand{\sample}{S}

\begin{document}

\begin{center}
  {\Large Element Level Differential Privacy: The Right Granularity of Privacy}
  
  \vspace{.3cm}
  
  \begin{tabular}{ccc}
    Hilal Asi$^1$\footnotemark
    & John C. Duchi$^{1,2}$ & Omid Javidbakht$^{2}$ \\
    \texttt{asi@stanford.edu} &
    \texttt{jduchi@stanford.edu} & 
    \texttt{omid\_j@apple.com} \vspace{0.5cm} \\
    \multicolumn{3}{c}{
      $^1$Stanford University ~~~ $^2$Apple
    }
  \end{tabular}
\end{center}

\footnotetext{Part of this work performed while in a summer
  internship at Apple. Partially supported by
  the Office of Naval Research award YIP N00014-19-2288.}
%
%

\input{abstract}

\input{introduction}

\input{definition}

\input{properties-of-element-dp}
\input{algorithms}

\input{statistical-learning}
\input{experiments}

\input{discussion}

\appendix

\input{quick-and-dirty}

\input{proof-heavy-hitters}
\input{proof-histogram}

\input{proof-loss-minimization}

\newpage

\setlength{\bibsep}{2pt}
\bibliography{bib}
\bibliographystyle{abbrvnat}



\end{document}

%% file: abstract.tex

\begin{abstract}
  Differential Privacy (DP) provides strong guarantees on the risk of
  compromising a user’s data in statistical learning applications, though
  these strong protections make learning challenging and may be too
  stringent for some use cases. To address this, we propose element level
  differential privacy, which extends differential privacy to provide
  protection against leaking information about any particular ``element'' a
  user has, allowing better utility and more robust results than classical
  DP. By carefully choosing these ``elements,'' it is possible to
  provide privacy protections at a desired granularity.
  We provide definitions, associated privacy
  guarantees, and analysis to identify the tradeoffs with the new
  definition; we also develop several private estimation and learning
  methodologies, providing careful examples for item frequency and
  M-estimation (empirical risk minimization) with concomitant privacy and
  utility analysis. We complement our theoretical and methodological
  advances with several real-world applications, estimating histograms and
  fitting several large-scale prediction models, including deep networks.
\end{abstract}

%% file: introduction.tex

\section{Introduction}

The substantial growth in data collection across many domains has led to
commensurate attention to and work on privacy risks in both
academic~\cite{DworkMcNiSm06, DworkRo14} and industrial
settings~\cite{ErlingssonPiKo14,ApplePrivacy17,BhowmickDuFrKaRo18}.
\citeauthor*{DworkMcNiSm06}'s \emph{differential
  privacy}~\cite{DworkMcNiSm06} and its variants~\cite{DworkKeMcMiNa06,
  BunSt16, Mironov17, DongRoSu19}---where a randomized algorithm returns
similar outputs for similar input samples---is now the standard privacy
methodology, as it gives provable protection against strong adversarial
attacks on privacy. Indeed, given the output of a differentially private
analysis on a sample $\sample = \{X_1, \ldots, X_n\}$, it is challenging to
identify whether a particular individual $x$ belongs to $\sample$ even for
an attacker knowing the entire sample except for a single observation.
These strong guarantees motivate work on private data analyses, including in
statistical estimation~\cite{Smith11, DuchiJoWa18}, machine
learning~\cite{ChaudhuriMoSa11}, game theory~\cite{McSherryTa07}, and
networks and graphs~\cite{KasiviswanathanNiRaSm13, KearnsRoWuYa16}.


Yet developing private algorithms that achieve reasonable utility is
challenging, as the strong protections differential privacy provides
necessarily degrade statistical utility. On the theoretical side, the
relative sample size necessary for private algorithms to achieve similar
utility to that of non-private algorithms grows with problem dimension and
inversely with the privacy parameter $\diffp$~\cite{BarberDu14a,
  SteinkeUl17, DuchiJoWa18, DuchiRo19}. On a practical level, this challenge
may lead privacy applications to instantiate a large privacy parameter
$\diffp$ to obtain acceptable statistical performance---for example,
\citet{AbadiChGoMcMiTaZh16} remarkably are able to fit neural
networks with differential privacy at all, though they
require a value of $\diffp = 8$
even for a weaker form of ``event level''
privacy to achieve performance approaching non-private algorithms---but
privacy guarantees for large values are unclear~\cite{DworkRo14}.


We argue that standard differential privacy's strong protections are not
always necessary to provide sufficient protection for a system's users.  For
example, an individual phone user sends multiple text messages, or takes
several cell-phone photos, each a single datum.  In such cases, it may be
satisfying from a privacy perspective not to protect whether a user
participates in a dataset---versions of differential privacy protect against
discovering this participation, though whether one has a phone is
likely not very sensitive---but to protect so that no one knows any
particular \emph{thing} a user has done, e.g., whether
the user has \emph{ever} typed a given word or taken a photo of a mountain.
Concretely, consider estimating the frequency of different word use
in email messages. Differential privacy prevents an attacker from
(accurately) distinguishing a user who sends hundreds of emails daily from
one who has never typed a word in his or her lifetime, a protection that may
be too strong. More nuanced tradeoffs can arise if we wish to
prevent an attacker from knowing, for example, whether a user has
ever typed a given word.

To address these challenges, we propose \emph{element-level differential
  privacy}, which aims to provide protection for what we---at the risk of
some hubristic excess---might term reasonable attacks.  The motivation for
our definition is that in many statistical estimation and learning problems,
an individual may contribute many datapoints; in a problem of learning from
mobile devices, a typical cell-phone contains many individual photos and
hundreds of distinct text messages, for example, and it is these
data that are private.
The key to
differential privacy and its descendant definitions is the notion of
\emph{neighboring datasets}~\cite{DworkRo14} or samples, where privacy
guarantees certify that an adversary given the output of a private mechanism
$\mechanism$ cannot reliably distinguish between its applications
$\mechanism(x)$ and $\mechanism(x')$ on neighboring samples $x$ and $x'$.
In differential privacy, two samples are neighboring if they differ in at
most a single observation.  As \citet{ChatzikokolakisAnBoPa13} note, it is
thus natural to quantify a distance between users or samples $x, x'$ to
redefine neighboring, and mechanisms then provide privacy for nearby users
under this distance~\cite{ChatzikokolakisAnBoPa13, AndresBoChPa13,
  BarberDu14a}.
\begin{figure}[ht]
  \begin{center}
      \begin{overpic}[width=0.9\columnwidth]{
	  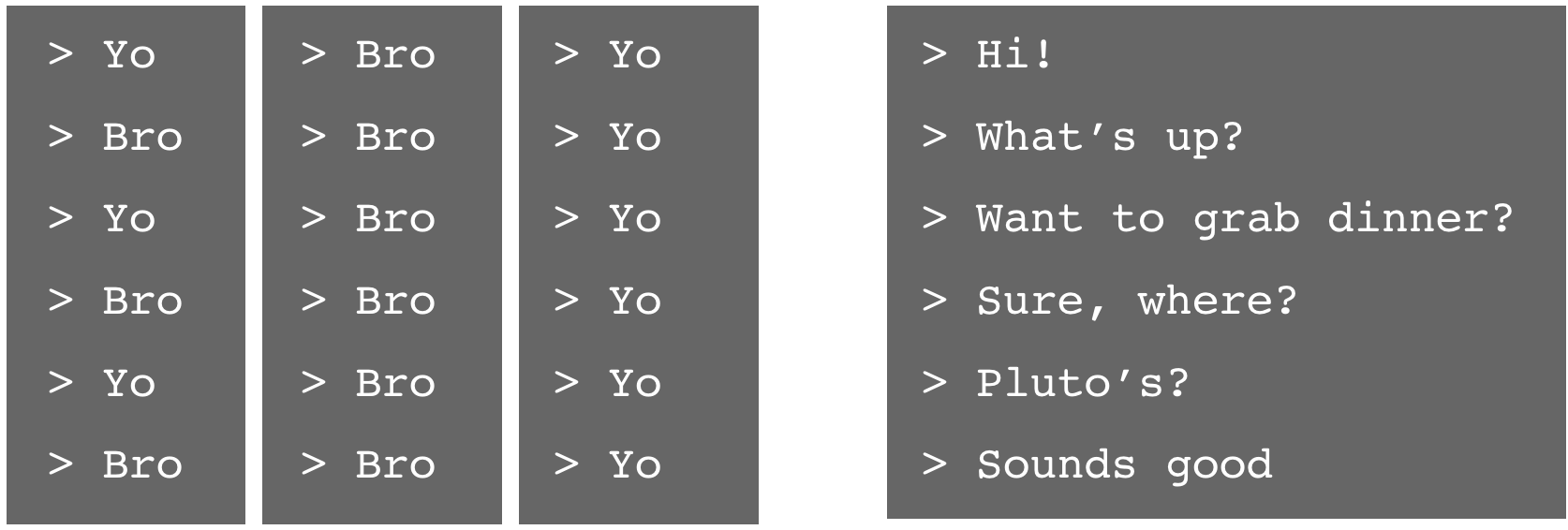}
      \end{overpic} 
      \caption{\label{fig:yo-bro} Example histories of four different users'
        text messages (each column represents a user's conversation).
        The left three
        columns reflect a conversation of the first author with his
        friends. The rightmost is a conversation between the second and
        third authors.  In the standard differential privacy definition, each
        user is distance $1$ from each other user.  In contrast,
        element-level privacy (with the histogram distance
        function described in the introduction)
        identifies the three left transcripts as neighboring---at
        distance 2---irrespective of the number of times each uses the
        word \texttt{yo} or \texttt{bro}, while the right conversation
        is distant.
        \vspace{-.5cm}
      }
  \end{center}
\end{figure}

Element-level privacy takes this idea and defines distances based on the
\emph{elements}, which we describe in the sequel, that an individual user's
data $x$ contains; here, two users are neighboring if they differ in one or
fewer elements. Consider estimating frequency of
word use in text (SMS) messages. Then a possible distance function between
two users is the \emph{number} of words that have different counts per user,
i.e., we represent each user as a vector $x \in \N^d$ of per-word counts
(how many times the user used each word in a dictionary of size $d$), and
the distance between users is the Hamming distance $d(x, x') = \sum_{j = 1}^d
\indics{x_j \neq x_j'}$ between their histograms (see
Figure~\ref{fig:yo-bro}). Element-level differential privacy then makes it
challenging for an attacker to discover any particular word a user
utters. In Section~\ref{sec:element-level} and throughout our applications,
we present more concrete examples to compare and contrast element-level and
classical differential privacy.


As we note above, there is substantial work on privacy broadly, with a line
of work investigating appropriate notions of distance and what distinctions
between individuals and data should be protected.  We highlight a few works
in this direction here. \citet{AndresBoChPa13} develop distance-based
notions of privacy to release information to geo-location services, where
privacy protections may degrade with distance to a user (e.g., it is
acceptable to release that a user is in Paris, but perhaps not at 28 Rue
Vieille du Temple). \citet{KasiviswanathanNiRaSm13} discuss protecting
privacy at the level of node differences in graphs, where two graphs are
neighboring if a single node is removed or added to the graph (with
arbitrarily many edges to other nodes), which is distinct from providing
privacy only on edge appearances. In the context of large-scale web or
mobile applications, there are differences between \emph{event-level}
privacy~\cite{DworkNaPiRo10, ErlingssonPiKo14, AbadiChGoMcMiTaZh16}, which
protects each individual action a user takes, though a user contributing
multiple data items (e.g.\ sending multiple text messages) suffers linear
degradation in privacy guarantees, and \emph{user-level}
privacy~\cite{McMahanRaTaZh18}, where all users are neighboring, no matter
how many data contributions they make or how diverse their data. The former
(event-level) provides limited privacy guarantees, while the latter
(user-level) may be too strong for practical use.  In this context,
\emph{element-level} privacy attempts to provide privacy at the right
granularity for the application at hand: in a way we formalize
shortly, one identifies the elements to be
protected, then guarantees that no matter how much data corresponding to a
particular element a user contributes, the output of the privacy mechanism
changes little.


In the remainder of the paper, we carefully define element-level
differential privacy (Section~\ref{sec:def}), using standard tools to show
that it inherits many of the desiderata important for satisfactory privacy
definitions (composition, group privacy, privacy to post-processing,
side-information resilience, and amplification by subsampling) in
Section~\ref{sec:element-DP-prop}. As one of our major goals is to provide
practicable procedures for estimation and learning with privacy protections,
we devote Section~\ref{sec:algo} to several methodological contributions. In
particular, we demonstrate histogram estimators and tools for estimation of
frequent elements, highlighting the advantages element-level privacy can
provide, and we show how to apply element-level privacy to fit large scale
machine learning models and compute M-estimators
(Section~\ref{sec:statistical-learning}) using stochastic-gradient-type
methods. Along the way, we demonstrate a new asymptotic normality result for
stochastic approximation procedures applied to fixed finite datasets, which
may be of interest beyond privacy.  We complement these with experimental
evidence on several real-world machine-learning tasks in
Section~\ref{sec:experiments}.

%% file: definition.tex

\section{Element-level privacy}
\label{sec:def}

\newcommand{\convexzo}{\mathsf{Conv}_{1 \downarrow 0}}
\newcommand{\permutations}{\Pi_n}
\newcommand{\permutation}{\pi}
    
As we allude in the introduction, our main goal in this paper is to provide
a new definition of privacy, simultaneously developing its properties while
demonstrating new procedures that obey its strictures.  To that end, we
begin by defining element-level privacy, contrasting it with prior notions.
The basic notion of privacy is \citeauthor{DworkMcNiSm06}'s
\emph{differential privacy} (DP), while other definitions of privacy, of
which we recapitulate a few, include approximate differential
privacy~\cite{DworkKeMcMiNa06}, R\'{e}nyi differential
privacy~\cite{Mironov17} and concentrated differential
privacy~\cite{DworkRo16, BunSt16}, and $f$-differential
privacy~\cite{DongRoSu19}.

\subsection{Privacy definitions}

The key to each of these definitions of privacy is a \emph{distance} on the
space of samples. In particular, let $\distsample : \mc{X}^n \times \mc{X}^n
\to \R_+$ be a distance on $\mc{X}^n$, and let $\mechanism$ be a randomized
mapping from $\mc{X}^n$ to some (measurable) space $\mc{Z}$.  In standard
differential privacy, this distance is the
(order-invariant) Hamming metric: letting $\permutations$ be the
collection of all permutations of $n$ elements,
for samples $\sample = (x_1, \ldots, x_n), \sample' =
(x_1', \ldots, x_n') \in \mc{X}^n$ we have
\begin{equation*}
  \distsample(\sample, \sample')
  = \dham(\sample, \sample')
  \defeq \min_{\permutation \in \permutations}
  \sum_{i = 1}^n \indics{x_i \neq x_{\permutation(i)}'}.
\end{equation*}
As \citet{ChatzikokolakisAnBoPa13} note, focusing on the case of
differential privacy, we may take any distance on the samples to provide
analogues of differential privacy; such alternative distances are important,
for example, for graph-based notions of differential
privacy~\cite{KasiviswanathanNiRaSm13}, location
services~\cite{AndresBoChPa13}, or event-level
streams~\cite{DworkNaPiRo10, ErlingssonPiKo14}.

We thus make the following definitions,
which generalize those in prior work by treating distance between two
samples as a first-class object.
\begin{definition}[Dwork et al.~\cite{DworkMcNiSm06,DworkKeMcMiNa06}]
  \label{definition:dp}
  Let $\diffp, \delta \ge 0$.  The randomized mechanism $\mechanism :
  \mc{X}^n \to \mc{Z}$ is \emph{$(\diffp, \delta)$-differentially private}
  for the distance $\distsample$ if for any pair of samples $\sample,
  \sample'$ with $\distsample(\sample, \sample') \le 1$ and any measurable
  subset $A \subset \mc{Z}$,
  \begin{equation*}
    \P(\mechanism(\sample) \in A)
    \le e^\diffp \P(\mechanism(\sample') \in A) + \delta,
  \end{equation*}
  where the probability is taken over only the randomness in
  $\mechanism$ (treating $\sample, \sample'$ as fixed).
\end{definition}
\noindent
We can abstract away from other definitions of privacy as well.
\begin{definition}[Mironov~\cite{Mironov17}]
  \label{definition:renyi-dp}
  Let $\diffp \ge 0,\renparam \ge 1$. The mechanism
  $\mechanism : \mc{X}^n \to \mc{Z}$ is \emph{$(\diffp,\renparam)$-R\'{e}nyi
    differentially private}
  for the distance $\distsample$ if for any pair of samples $\sample,
  \sample'$ with $\distsample(\sample, \sample') \le 1$,
  \begin{equation*}
    \drenyi{\mechanism(\sample)}{\mechanism(\sample')} \le \diffp.
  \end{equation*}
\end{definition}
\noindent
\citet{Mironov17} shows that any $(\diffp,\renparam)$-R\'{e}nyi private
mechanism is also $(\diffp + \frac{\log{\frac{1}{\delta}}}{\alpha - 1},
\delta)$-differential private for all $\delta \in [0,1]$. As a
consequence, if we wish to compute and release
$f(\sample)$ for some $\sample \in \mc{X}^n$, and
$\ltwos{f(\sample) - f(\sample')} \le \rho$ for any neighboring
samples $\sample, \sample'$, then the Gaussian mechanism
\begin{equation}
  \label{eqn:mironov-gaussian-mech}
  \mechanism(\sample) \defeq f(\sample) + \normal\left(0,
  \rho^2 \left(\frac{\indic{\diffp > 1}}{\diffp}
  + \frac{2 \log(1/\delta)}{\diffp^2}
  \right) I\right)
\end{equation}
provides $(\diffp, \delta)$-differential privacy for all $\diffp > 0$. For
$\diffp \le 1$, the $1/\diffp$ term in the normal variance is
unnecessary~\cite{DworkKeMcMiNa06}. (See
Appendix~\ref{sec:sufficiency-of-mironov} for this calculation.)

Rather than exhaustively discussing alternative privacy
definitions, we note that each variant of differential
privacy ($f$-differential privacy~\cite{DongRoSu19}
or concentrated differential privacy~\cite{DworkRo16,BunSt16})
similarly rely on sample distances, saying that a mechanism
$\mechanism(\cdot)$ is private if its output distribution changes
little (under an appropriate metric) when its input sample changes.

\subsection{Element-level privacy definition}
\label{sec:element-level}

\providecommand{\sups}[1]{^{(#1)}}
\newcommand{\distelement}{d_{\mathsf{element}}}

The standard distance in each privacy definition is the Hamming
distance between samples $\sample, \sample'$; this is satisfying, as it
limits any inferences that can be made about an
individual~\cite{DworkMcNiSm06, Dwork08}. In some scenarios, this definition
makes learning challenging (or, depending on the task and desired privacy
guarantee, essentially impossible)~\cite{DuchiJoWa18, DuchiRo19}.  It is
thus natural to consider more fine-grained distance notions to allow utility
while providing sufficient privacy. For our purposes, it is useful to
consider a scenario frequent in large-scale learning applications, such as
federated learning (e.g.~\cite{AbadiChGoMcMiTaZh16}), where individual users
contribute multiple data items rather than a single item.
In such cases,
we protect a user so that no one knows any particular \emph{thing}
the user has done. For example,
a student with a phone sends many text messages, but
may wish that his parents and teachers never know whether he has ever sent a
curse word, irrespective of the number of times he may or
may not have sent one.


To formalize this, we introduce \emph{element-level privacy}. A sample or
dataset $\sample$ consists of $n$ user's data (or data units) $\sample =
\{x\sups{u}\}_{u = 1}^n$, while each user $u$ maintains local data of size
$m(u)$, where the size may depend on the user $x\sups{u} = \{x_1\sups{u},
\ldots, x_{m(u)}\sups{u}\}$. For example, individual $u$'s data may consist
of the $m(u)$ photos she has taken.  External to the users are $K$
\emph{clusters} $\{c_1, \ldots, c_K\}$ partitioning $\mc{X}$,
where we view the cluster centroids
as the \emph{elements} to be made private, and each datapoint
$x\sups{u}_i$ belongs to precisely one cluster $c_k$ (i.e.\ has
a nearest element); we denote this by
$x\sups{u}_i \in c_k$.  The distance between two users' local data $x =
\{x_1, \ldots, x_n\}$ and $x' = \{x_1', \ldots, x_m'\}$ is then the number
of clusters $c_1, \ldots, c_K$ with different memberships for the two users'
data, that is,
\begin{equation}
  \label{eqn:def-user-dist}
  \begin{split}
    \distuser(x, x')
    & = \distuser(\{x_1, \ldots, x_n\}, \{x_1', \ldots, x_m'\}) \\
    & \defeq
    \sum_{k = 1}^K
    \indic{
      \{x_i : x_i \in c_k\} \neq \{x_i' : x_i' \in c_k\}
    },
  \end{split}
\end{equation}
where $\{x_i : x_i \in c_k\}$ are implicitly multi-sets.
Then two users' data $x, x'$ are \emph{element-neighbors} if $\distuser(x,
x') \le 1$; this is equivalent to allowing users to differ
arbitrarily on one element of their data.
With this distance definition, we can then define the element-level
sample distance by
\begin{equation}
  \label{eqn:element-distance}
  \distelement(\sample, \sample')
  \defeq \min_{\permutation \in \permutations}
  \sum_{u = 1}^n \distuser(x\sups{u}, {x'}\sups{\permutation(u)}).
\end{equation}
Two samples $\sample,
\sample'$ of size $n$ are \emph{element-neighbors} if each of the units
within the sample is identical except for (at most) one unit $x \in \sample,
x' \in \sample'$, where $\distuser(x, x') \le 1$.
The definition of element level privacy is now immediate: we
take the sample distance $\distsample$ in any
privacy definition
(e.g.\ \ref{definition:dp} or~\ref{definition:renyi-dp}) to
be $\distelement$.
\begin{definition}
  \label{definition:element-dp}
  A mechanism $\mechanism$ satisfies \emph{element-level}
  differential
  privacy or R\'{e}nyi-differential privacy if it satisfies
  Definition~\ref{definition:dp} or \ref{definition:renyi-dp},
  (respectively) with
  distance $\distsample = \distelement$.
\end{definition}

Element-level differential privacy guarantees that the releases of a
mechanism trained on users' sensitive data does not leak any particular
``element'' the user has, that is, whether a user has data belonging to any
one of the clusters $c_1, \ldots, c_K$, no matter how many data point belong
to one of the clusters. It is useful to compare this
definition to two frequent definitions of privacy for large-scale learning
systems.  The first is \emph{event-level privacy}~\cite{ErlingssonPiKo14},
which applies privacy commensurate with each individual \emph{event} a user
performs, for example, whenever a user visits any website. This definition
may be too weak: consider a user who sends 50 text-messages consisting of
the phrase ``Hello!'' Then event-level
privacy (say with Def.~\ref{definition:dp}) guarantees
a likelihood ratio bound of
$e^{50 \diffp}$ versus an otherwise identical user who never uses the phrase
``Hello!'' In the case of element-level privacy, however, the distance
between these users is at most 1 regardless of how many times either says
``Hello!''  The second common definition is \emph{user-level privacy}, which
corresponds to the standard definitions with Hamming distance; by taking a
single cluster $c_1 = \mc{X}$ in the
definitions~\eqref{eqn:def-user-dist}--\eqref{eqn:element-distance} of
element level distances, one recovers user-level privacy, but as we shall
see, the additional flexibility of element-level privacy allows more
utility.

To get a feel for Definition~\ref{definition:element-dp}, it is instructive
to consider two (somewhat stylized) examples.

\begin{example}[Word frequency estimation]
  \label{example:frequent-text-words}
  Consider the problem of estimating frequent words used in text (SMS)
  messages. Ignoring punctuation, we treat each word as a cluster, so that
  for a dictionary of size $d$, a user $u$'s data $x\sups{u} =
  \{x_1\sups{u}, \ldots, x_d\sups{u}\}$ consists of the counts $x_j\sups{u}
  \in \N$ of the times user $u$ typed word $j$, a histogram of word
  counts. In Figure~\ref{fig:yo-bro}, for example,
  the leftmost column has histogram with count 3 for the word
  ``yo,'' 3 for ``bro,'' and 0 for all other words.
  The distance between two user data $x, x'$ is then $\distuser(x,
  x') = \sum_{j = 1}^d \indics{x_j \neq x_j'}$, the number of distinct
  counts. In this case, two users are neighboring when their word use is
  identical except that one may use a word $j$ arbitrarily more or less than
  the other.
\end{example}

\begin{example}[Website visit counts]
  \label{example:urls}
  Consider estimating the frequency of popular websites (URLs) that users
  visit. In this case, a natural set of elements are domains (the first part
  of a website name), while specific URLs belong to a single domain.  For
  example, \url{https://en.wikipedia.org/wiki/Apple_Inc.} and
  \url{https://en.wikipedia.org/wiki/NeXT} belong to the domain (cluster)
  \texttt{wikipedia.org}, while \url{http://web.stanford.edu/~jduchi/} and
  \url{http://web.stanford.edu/~asi/} belong to \texttt{stanford.edu}.  Then
  a user's data consists of all URLs he or she visits, while the distance
  between users is the number of domains in which they visit distinct URLs.
  The intuition here is that any mechanism satisfying
  Definition~\ref{definition:element-dp} limits release of whether a user
  ever even visits a website in a particular domain, for example,
  \texttt{wikipedia.org}, \texttt{stanford.edu}, or \texttt{youtube.com}.
  In contrast, standard differential privacy would protect
  whether
  a user has ever used the internet.
\end{example}


As these examples attempt to clarify, the
important facet of element-level DP is that it protects a data provider from
anyone ever knowing any particular thing they have done, regardless of how
many times they have done it: visiting a domain, using a word, or other
desired protected element.


%% file: properties-of-element-dp.tex

\newcommand{\subwithout}{\Pi_m^{\textup{wo}}}
\newcommand{\subfixed}{\Pi_q}

\subsection{Properties of element-level differential privacy}
\label{sec:element-DP-prop}

By replacing the standard Hamming distance in the different definitions of
differential privacy with the element-based
distance~\eqref{eqn:element-distance}, any element-level differentially
private mechanism inherits the typical properties private mechanisms enjoy,
including privacy to post-processing, group privacy, composition, and
amplification of privacy by (anonymous) subsampling (see the
book~\cite{DworkRo14} for a discussion of these desiderata).  Almost all of
these inheritances are immediate, but to give
a flavor of these results we present several
for the $(\diffp, \delta)$-element-level
differentially private case.

\begin{corollary}[Post-Processing]
  \label{corollary:post-processing}
  Let $\mechanism : \mc{X}^n \to \mc{Z}$ be an
  $(\diffp,\delta)$-element-level private mechanism. For any (potentially
  randomized) function on $\mc{Z}$, the mechanism $f \circ
  \mechanism$ is $(\diffp,\delta)$-element-level private.
\end{corollary}
\begin{corollary}[Group Privacy]
  \label{corollary:group-privacy}
  Let $\mechanism : \mc{X}^n \to \mc{Z}$ be an $(\diffp,\delta)$-element-level
  private mechanism.  Let
  $\sample = \{x\sups{u}\}_{u=1}^n$ and
  $\sample' = \{{x'}\sups{u}\}_{u=1}^n \in \mc{X}^n$ be
  two samples. Then for any measurable set $A$,
  \begin{equation*}
    \P(\mechanism(\sample) \in A)
    \le e^{\distelement(\sample, \sample')\diffp}
    \P(\mechanism(\sample') \in A) +
    \distelement(\sample, \sample')
    e^{(\distelement(\sample, \sample') - 1) \diffp}
    \delta.
  \end{equation*}
\end{corollary}
\noindent
See, for example, \citet[Prop.~2.1, Thm.~2.2]{DworkRo14}.  We also
immediately have composition for element-level
DP. In this case, we consider adaptive composition of $k$ mechanisms,
where for each $i$, we assume the output space $\mc{Z}_i$ is a measurable space
and
\begin{equation*}
  \mechanism_i : \mc{X}^n \times \mc{Z}_1 \times \cdots \times \mc{Z}_{i-1}
  \to \mc{Z}_i
\end{equation*}
is $(\diffp_i, \delta_i)$-element-level differentially private, meaning that
for fixed $z_1^{i-1}$, $\mechanism_i(\cdot, z_1^{i-1})$ is private. The
$k$-fold composition $\mechanism_k \circ \cdots \circ \mechanism_1$ then has
recursive definition $Z_i = \mechanism_i(\sample, Z_1^{i-1})$. We have the
following corollary
(for the proof of differentially private version,
see~\cite[Thm.~3.20 and Appendix~B]{DworkRo14} and
\cite[Corollary 6.26]{Duchi19}, and for the R\'{e}nyi version,
see \cite[Prop.~1]{Mironov17}).
\begin{corollary}[Composition]
  \label{corollary:composition}
  Let $k \in \N$, $\mc{Z}_i$ be measurable spaces,
  and $\mechanism_i : \mc{X}^n \times \mc{Z}_1^{i-1} \to \mc{Z}_i$
  be
  $(\diffp_i, \delta_i)$-element-level DP.
  Then their $k$-fold composition
  is $(\sum_{i = 1}^k \diffp_i, \sum_{i = 1}^k \delta_i)$-element-level
  DP. Additionally, for any $\delta_0 > 0$,
  the composition is
  \begin{equation*}
    \left(\frac{3}{2} \sum_{i = 1}^k \diffp_i^2
    + \sqrt{6 \sum_{i = 1}^k \diffp_i^2 \log \frac{1}{\delta_0}},
    \delta_0 + \sum_{i=1}^k \frac{\delta_i}{1 + e^{\diffp_i}}\right)
  \end{equation*}
  element-level DP.
  If instead the mechanisms are $(\diffp_i, \renparam)$-element-level
  R\'{e}nyi private, the composition
  is $(\sum_{i = 1}^k \diffp_i, \renparam)$-element-level R\'{e}nyi private.
\end{corollary}

We also obtain that subsampling amplifies the privacy of our mechanisms
(see~\cite[Thms.\ 8 \& 9]{BalleBaGa18}). We consider the two most
natural subsampling mechanisms. The first, which we denote
$\subfixed$, takes a sample
$\sample$ and returns a subsample $\sample' \subset \sample$ where
each element $x\sups{u} \in \sample$ is included with a fixed probability
$q \in (0, 1)$. The second, $\subwithout$, samples $m$ elements
without replacement from $\sample$.
\begin{corollary}[Amplification by subsampling]
  \label{corollary:subsampling}
  Let $\mechanism$ be an $(\diffp,\delta)$-element differentially private
  mechanism that acts on samples of arbitrary size. Then
  \begin{enumerate}[label=(\roman*)]
  \item For any $q \in (0, 1)$,
    the subsampled mechanism $\mechanism \circ \subfixed$ is $(\log(1 + q
    (e^\diffp - 1)), q \delta)$-element-level differentially private.
  \item For any $m \le n \in \N$, the subsampled mechanism $\mechanism \circ
    \subwithout$ applied to samples of size $|\sample| = n$
    is $(\log(1 + \frac{m}{n}(e^\diffp - 1)), \frac{m}{n}
    \delta)$-element-level differentially private.
  \end{enumerate}
\end{corollary}

Finally, we discuss amplification of R\'{e}nyi element-level differential
privacy by subsampling using a particular Gaussian mechanism that will
form the basis for our stochastic approximation results in the sequel.
In this case, we build off of \citeauthor{AbadiChGoMcMiTaZh16}'s
\emph{moments accountant}~\cite{AbadiChGoMcMiTaZh16}, whose primitive
is to release a sum of vectors.
Consider samples of the form
$\sample = \{x^{(1)}, \ldots,  x^{(n)}\}$, where
each $x\sups{u}$ consists of a collection
$x\sups{u} = \{x_k\sups{u}\}_{k=1}^{K}$ of $K$ vectors, where each 
$x_k\sups{u}$ corresponds to a desired statistics for 
element/cluster $k$ and each individual vector satisfies $\ltwos{x_k\sups{u}} \le \rho$
for some $\rho < \infty$. The goal is to release a sum of
the entire sample, $\sum_{u = 1}^n \sum_{k = 1}^{K} x_k\sups{u}$,
but instead we consider subsampling by users. In particular,
for $q \in [0, 1]$ and $\sigma \ge 0$,
let $B_u \in \{0, 1\}$ be either
i.i.d.\ $\bernoulli(q)$ or uniform on $\sum_u B_u = q n$,
let $W \sim \normal(0, I)$, and consider
the mechanism
\begin{equation}
  \label{eqn:subsampled-sum}
  \mechanism(\sample) \defeq \sum_{u=1}^n B_u \bigg(
  \sum_{k = 1}^{K} x^{(u)}_k\bigg) + \rho \sigma W.
\end{equation}
Taking $\distuser(x, x')$ as the user
distance~\eqref{eqn:def-user-dist}, we have the following corollary,
with minor extension to handle
the variants of subsampling (i.i.d.\ or fixed size without
replacement).
\begin{corollary}[Moments accountant, \cite{AbadiChGoMcMiTaZh16}
    Lemma 3]
  \label{corollary:momentacc}
  Let $\renparam \ge 1$ and $P_t$ be the $\normal(t,
  \sigma^2)$ distribution.  The mechanism~\eqref{eqn:subsampled-sum} is
  $(\diffp_\renparam, \renparam)$-element-level R\'{e}nyi differentially
  private with
  \begin{align*}
    \lefteqn{\diffp_\renparam(q, \sigma) \defeq} \\
    & ~~~ \max\left\{
    \drenyi{q P_1 + (1 - q) P_0}{q P_{-1} + (1 - q) P_0},
    \drenyi{q P_{-1} + (1 - q) P_0}{q P_{1} + (1 - q) P_0}
    \right\}.
  \end{align*}
\end{corollary}

It is possible to numerically evaluate the R\'{e}nyi divergences in the
corollary, making them effective in applications, though they are
unavailable analytically.  As a consequence of the corollary, if we compose
the mechanism~\eqref{eqn:subsampled-sum} adaptively $T$ times, composition
for R\'{e}nyi privacy immediately guarantees the entire mechanism is $(T
\diffp_\renparam(q, \sigma), \renparam)$-R\'{e}nyi element-level
private. As a consequence, recalling \citeauthor{Mironov17}'s transformation
from R\'{e}nyi to approximate differential privacy~\cite{Mironov17}, for any
$\delta > 0$ the same composition is also $(\diffp, \delta)$-element-level
DP for $\diffp = \inf_{\renparam \ge 1} T \diffp_{\renparam}(q,
\sigma) + \frac{\log \frac{1}{\delta}}{\renparam - 1}$.  To give a sense of
the level of privacy maintained, we consider the bound of
\cite[Lemma 3]{AbadiChGoMcMiTaZh16}; a slight variant of its proof
yields
\begin{equation*}
  \diffp_\renparam(q, \sigma)
  \le \frac{q^2 \renparam}{1 - q} \frac{1}{\sigma^2}
  + O\left(q^3 / \sigma^3\right)
\end{equation*}
for $\renparam \le \sigma^2 \log \frac{1}{q \sigma}$ and $\sigma \ge 2$.
Thus, for numerical constants $c_0, c_1$, this composition is $(\diffp,
\delta)$-element level private for $\diffp \le c_0 q^2 T$ and $\sigma^2 \ge
c_1\frac{q^2 T}{\diffp^2} \log \frac{1}{\delta}$.

We remark in passing and without proof that each of the preceding
corollaries has an analog in \citeauthor{DongRoSu19}'s $f$-differential
privacy~\cite{DongRoSu19}.

%% file: algorithms.tex
\section{Element-level private methods}
\label{sec:algo}

\newcommand{\globalsens}{\mathsf{gs}}
\newcommand{\elementsens}{\mathsf{es}}

One of our major goals is to demonstrate the methodological possibilities of
mechanisms satisfying element-level privacy, both to give some sense of the
way to design mechanisms satisfying the definition and to understand the
potential utility benefits---in terms of more accurate
estimation---element-level privacy allows over user-level notions of
privacy.  To that end, we present three examples in this section of
increasing sophistication: discovering most frequent elements or heavy
hitters (Sec.~\ref{sec:heavy-hitters}), estimating multinomial frequencies
(Sec.~\ref{sec:histogram}), and finally, stochastic optimization and
statistical learning (Sec.~\ref{sec:statistical-learning}).

We begin by attempting to give a somewhat general picture, connecting to the
classical Laplace mechanisms and sensitivity analyses of
\citet{DworkMcNiSm06}; we specialize in the coming sections. Suppose each
user contributes a batch $x = (x_1, \ldots, x_m)$ of data, and we wish to
compute the average $\frac{1}{n} \sum_{u = 1}^n f(x\sups{u})$ of a function
$f : \mc{X}^m \to \R$ on $\sample = \{x\sups{u}\}_{u = 1}^n$. Standard
mechanisms add noise that scales with the \emph{global
  sensitivity} of the function $f$, that is, $\globalsens(f) \defeq \sup_{x
  \in \mc{X}^m, x' \in \mc{X}^m} |f(x) - f(x')|$, and the Laplace
(respectively Gaussian) mechanisms for $\diffp$- or $(\diffp,
\delta)$-differentially private release are
\begin{equation*}
  \mechanism(\sample)
  \defeq \frac{1}{n} \sum_{u = 1}^n f(x\sups{u})
  + \frac{\globalsens(f)}{\diffp} \cdot \laplace(1)
  ~~ \mbox{and} ~~
  \mechanism(\sample)
  \defeq \frac{1}{n} \sum_{u = 1}^n f(x\sups{u})
  + \globalsens(f) \cdot \normal\left(0, \frac{\indic{\diffp > 1}}{\diffp}
  + \frac{2 \log \frac{1}{\delta}}{\diffp^2}\right).
\end{equation*}
In contrast, given a partition $\{c_1, \ldots, c_K\}$ of $\mc{X}$ and
corresponding user distance $\distuser$ (recall
Eq.~\eqref{eqn:def-user-dist}), the analogous recipe here is to add noise
scaling with the \emph{element sensitivity} of $f$,
\begin{equation}
  \label{eqn:element-sensitivity}
  \elementsens(f) \defeq \sup_{x \in \mc{X}^m,
    x' \in \mc{X}^m} \left\{|f(x) - f(x')|
  ~ \mbox{s.t.}~ \distuser(x, x') \le 1 \right\},
\end{equation}
which satisfies $\elementsens(f) \le \globalsens(f)$.
Then the standard Laplace and Gaussian mechanisms become
\begin{equation}
  \label{eqn:basic-mechanisms}
  \mechanism(\sample)
  \defeq \frac{1}{n} \sum_{u = 1}^n f(x\sups{u})
  + \elementsens(f) \cdot
  \begin{cases}
    \laplace(1 / \diffp) & \mbox{Laplace mechanism}
    \\
    \normal\left(0, \frac{\indic{\diffp > 1}}{\diffp}
    + \frac{2 \log(1/\delta)}{\diffp^2} \right) & \mbox{Gaussian
      mechanism}
  \end{cases}
\end{equation}
and guarantee $\diffp$- or $(\diffp,\delta)$-element-level differential
privacy.
We see utility gains whenever $\elementsens(f) \ll
\globalsens(f)$, which we expect when the number $K$ of elements is large,
providing finer granularity privacy.







\subsection{Discovering heavy hitters}
\label{sec:heavy-hitters}

The first two examples we consider are to estimate properties of a
multinomial.  We consider a sampling scheme where each of $n$ users
generates a vector $X\sups{u} \in \N^d$, $X \sups{u} \simiid \multinomial(m,
p)$, where $p = (p_1, \ldots, p_d) \in \R^d_+$ is an unknown vector of
probabilities, $p^T \ones = 1$, and $m \in \N$ is the number of trials. The
goal is to estimate different properties of the vector $p$, where
$X\sups{u}_j \in \{0, \ldots, m\}$ indicates the count of appearances of
item $j$ for user $u$. For example, if these multinomials indicate purchases
users make in a grocery store, we may be interested in the items the most
users purchase.  \citet{BhaskarLaSmTh10} provide a sophisticated analysis of
private algorithms for finding multisets of frequent items in item stream;
we consider a much simpler scenario than theirs (we care only about
individual items/elements, and wish to release an ordering of all elements
rather than a top few, the latter adding significant complexity to the
problem), as we treat this more as an illustrative example.  While stylized,
it is illustrative of the approaches possible with element-level privacy.

We assume that each item $j \in \{1, \ldots, d\}$ is
an element, so that the user distance $\distuser(x, x') = \sum_{j = 1}^d
\indics{x_j \neq x_j'}$. Then we consider the mechanism
\begin{equation}
  \label{eqn:heavy-hitters-mech}
  \ones(x) \defeq [\indic{x_j > 0}]_{j = 1}^d,
  ~~~
  \what{H} \defeq \mechanism(\sample)
  \defeq \sum_{u = 1}^n \ones(X\sups{u})
  + \normal\left(0, \sigma^2 I_d\right)
\end{equation}
for some $\sigma \ge 0$ to be chosen depending on the desired privacy. 
Following our discussion to begin Sec.~\ref{sec:algo},
the element sensitivity~\eqref{eqn:element-sensitivity} of $\ones(x)$ is
$\elementsens(\ones(\cdot)) = \sup \{\ltwos{\ones(x) - \ones(x')} :
\distuser(x, x') \le 1\} = 1$, so the following is immediate by
Definition~\ref{definition:renyi-dp} and the Gaussian
mechanism~\eqref{eqn:mironov-gaussian-mech}.
\begin{lemma}
  \label{lemma:privacy-heavy-hitters}
  Let $\diffp \ge 0$ and assume that each observation
  $X\sups{u}$ satisfies $\sum_j X_j\sups{u} = m$ as above.
  The mechanism~\eqref{eqn:heavy-hitters-mech} provides the following
  privacy guarantees.
  \begin{enumerate}[label=(\roman*)]
  \item Let $\renparam \ge 1$ and take $\sigma^2 = \frac{\renparam}{\diffp}$.
    Then $\mechanism$ is $(\diffp, \renparam)$-element level
    R\'{e}nyi private.
  \item Let $\delta \in (0, 1)$ and $\sigma^2 = \frac{\indics{\diffp >
      1}}{\diffp} + \frac{2 \log \frac{1}{\delta}}{\diffp^2}$. Then
    $\mechanism$ is $(\diffp, \delta)$-element level differentially private.
  \end{enumerate}
\end{lemma}
\noindent
In contrast to the element-level noise scaling above, the
global sensitivity of the indicator vector $\ones(\cdot)$ is
$\globalsens(\ones(\cdot)) =
\sup\{\ltwos{\ones(x) - \ones(x')} : \lones{x} \le m,
\lones{x'} \le m\} = \sqrt{2m}$,
so that noise addition mechanisms
for standard differential privacy
(e.g.\ the Gaussian mechanism~\eqref{eqn:mironov-gaussian-mech})
add noise whose variance on each coordinate scales
as
\begin{equation*}
  \sigma_{\textup{std}}^2 \defeq
  m \cdot \left(\frac{\indic{\diffp > 1}}{\diffp}
  + \frac{2 \log \frac{1}{\delta}}{\diffp^2} \right)
\end{equation*}
to achieve $(\diffp, \delta)$-differential privacy.


\newcommand{\heavyloss}{L_{\textup{order}}}

Rather than attempting to recover the actual frequencies of
appearance, we consider a loss measuring the number of mis-ordered
pairs of elements, an estimate suffers loss if it mis-orders a pair
of indices $(i, j)$ where $p_i \ge p_j + \gamma$ for some threshold
$\gamma$. We assume w.l.o.g.\ that $p_1 \ge p_2 \ge \ldots \ge p_d$,
and define
\begin{equation}
  \label{eqn:heavyhittersloss}
  \heavyloss(\what{H})
  \defeq \E\bigg[\sum_{j = 1}^{d-1} \sum_{l=j+1}^d
    \indic{p_j - p_l \ge \gamma} \indics{\what{H}_j > \what{H}_l}
  \bigg].
\end{equation} 
We then have the following proposition, whose proof we
provide in Appendix~\ref{proof:heavyhitters}.
\begin{proposition}
  \label{proposition:heavyhitters}
  Assume that $p_1 = \max_j p_j \le \frac{1}{2m}$.
  Let $0 \le t \le d$ and $\what{H}$ denote the
  mechanism~\eqref{eqn:heavy-hitters-mech} with
  Gaussian noise. Then $\heavyloss(\what{H}) \le t^2$ whenever
  \begin{equation*}
    \gamma \ge \max\left\{\frac{32}{nm} \log \frac{d}{t},
    \frac{4 \sqrt{2}\sigma}{nm} \sqrt{\log \frac{d}{t}},
    \frac{12 \sqrt{2}}{\sqrt{5}}
    \cdot \frac{\sqrt{p_1}}{m \sqrt{n}} \sqrt{\log \frac{d}{t}}
    \right\}.
  \end{equation*}
\end{proposition}

We provide a bit of commentary on this result. First, we consider the
scaling to achieve a fixed error $\heavyloss(\what{H}) \le t^2$, where the
dominant terms are the second two in the maximum of
Proposition~\ref{proposition:heavyhitters}. Let
$\gamma_{\textup{el}}$ denote the separation threshold at which we obtain
small loss for element-level privacy and $\gamma_{\textup{std}}$
that for the mechanism providing standard $(\diffp, \delta)$-differential
privacy (i.e.\ mechanism~\eqref{eqn:heavy-hitters-mech} with
variance $\sigma_\textup{std}^2$). Then ignoring logarithmic factors,
we require
\begin{equation*}
  \gamma_{\textup{el}} \gtrsim
  \frac{1}{\diffp} \frac{1}{nm} \vee
  \frac{\sqrt{p_1}}{m \sqrt{n}}
  ~~~ \mbox{while} ~~~
  \gamma_{\textup{std}}
  \gtrsim \frac{1}{\diffp n \sqrt{m}} \vee
  \frac{\sqrt{p_1}}{m \sqrt{n}}.
\end{equation*}
Thus, in high dimensional situations where we expect that $p_1$ is
small enough that $p_1 \ll \sqrt{m / n}$, the element-level private
mechanism can provide substantially fewer ordering errors than
a mechanism providing user-level privacy.

\subsection{Histogram estimation}
\label{sec:histogram}

\newcommand{\clusterproj}{\pi_{\rho, \{c_k\}}}

We now turn to the problem of estimating item frequencies---histogram
estimation---one of the original motivations for differential
privacy~\cite[Ex.~3]{DworkMcNiSm06}.
We are in an identical setting to Sec.~\ref{sec:heavy-hitters},
where $X\sups{u} \simiid \multinomial(m, p)$ for some $m \in \N$ and
$p \in \R^d_+$ with $p^T \ones = 1$.
We elaborate this setting somewhat to allow more substantial
elements, as in Example~\ref{example:urls}, by assuming there
are $K$ \emph{clusters} $\{c_1, \ldots, c_K\}$ partitioning
$[d]$. For shorthand, for $v \in \R^d$ we let
$v_{c_k} = [v_j]_{j \in c_k} \in \R^{|c_k|}$, and
we denote the probability of an item in $c_k$
by $P(c_k) = \ones^T p_{c_k} = \sum_{j \in c_k} p_j$.
  
We consider a normal noise addition mechanism~\eqref{eqn:basic-mechanisms},
but our first step is to design a function insensitive
to changes within the partition $\{c_1, \ldots, c_K\}$ of $[d]$, reducing
the element sensitivity. To that end, we consider a mechanism that first
projects each cluster $c_k$ of counts into an $\ell_2$-ball, then adds
Gaussian noise.  For $v \in \R^d$, we define the projection
\begin{equation*}
  \clusterproj(v) \defeq
  \argmin_{x \in \R^d} \left\{ \ltwo{x - v}^2
  : \ltwo{x_{c_k}} \le \rho\right\}
  = \big[v_{c_k} \cdot \min\left\{1, \rho / \ltwos{v_{c_k}}
    \right\}\big]_{k = 1}^K
\end{equation*}
(with the obvious re-ordering in the second equality).
The mechanism is then
\begin{equation}
  \label{eqn:histogram-mechanism}
  \mechanism(\sample, \rho, \{c_k\})
  \defeq
  \frac{1}{n} \sum_{u=1}^n \clusterproj(X\sups{u})
  + \normal\left(0, \frac{\rho^2 \sigma^2}{n^2} I_d\right).
\end{equation}
As with Lemma~\ref{lemma:privacy-heavy-hitters},
we then immediately obtain the privacy of the
mechanism~\eqref{eqn:histogram-mechanism}.
\begin{lemma}
  \label{lem:dp-truncation-cluster}
  Let $\diffp \ge 0$ and assume that each observation $X\sups{u}$
  satisfies $\sum_j X_j\sups{u} = m$ as above.
  The mechanism~\eqref{eqn:histogram-mechanism} provides the following
  privacy guarantees.
  \begin{enumerate}[label=(\roman*)]
  \item Let $\renparam \ge 1$ and take $\sigma^2 = \frac{\renparam}{\diffp}$.
    Then $\mechanism$ is $(\diffp, \renparam)$-element-level
    R\'{e}nyi private.
  \item Let $\delta \in (0, 1)$ and $\sigma^2 = \frac{\indic{\diffp >
      1}}{\diffp} + \frac{2 \log \frac{1}{\delta}}{\diffp^2}$. Then
    $\mechanism$ is $(\diffp, \delta)$-element-level differentially private.
  \end{enumerate}
\end{lemma}

We now turn to an investigation of the error of
the mechanism~\eqref{eqn:histogram-mechanism}, providing
the following proposition
(whose proof we give in Appendix~\ref{proof:element-level-mse-clustering}).
\begin{proposition}
  \label{proposition:element-level-histogram}
  Let $m \ge 3, t \ge 0$, and assume that for cluster probabilites $P(c) =
  \sum_{j \in c} p_j$ we have $\rho \ge \min\{3 m P(c) + 3 \log m + t, m\}$
  for each $c \in \{c_k\}$. Then there exists $q \in \R^d_+$ with $\ones^T
  q_c \le P(c)$ for each $c \in \{c_k\}$ and a numerical constant $C > 0$
  such that for each $j \in [d]$,
  \begin{equation*}
    \P\left(\left|\mechanism_j(\sample, \rho, \{c_k\}) - m p_j\right|
    \ge 2^{2 - t} q_j + u \right)
    \le \exp\left(-C \min\left\{\frac{n u^2}{m p_j},
    \frac{n^2 u^2}{\sigma^2 \rho^2},
    \frac{n u}{\rho} \right\}\right)
  \end{equation*}
  for all $u \ge 0$. In addition, for numerical constants $C_0 \le C_1 <
  \infty$,
  \begin{equation*}
    C_0 \left[\frac{m p_j}{n}
      + \frac{\sigma^2 \rho^2}{n^2}\right]
    \le
    \E\left[|\mechanism_j(\sample, \rho, \{c_k\}) - m p_j|^2\right]
    \le C_1 \left[2^{-2t} q_j^2 +
      \frac{m p_j}{n} +
      \frac{\sigma^2 \rho^2}{n^2}\right].
  \end{equation*}
  If $\rho \ge m$, the preceding inequalities hold with $t = \infty$.
\end{proposition}

Let us compare standard mechanism's errors with the
element-level mechanism's errors, focusing on the squared error. For
the user-level case, we have global sensitivity
$\rho = m$, and the proposition shows
that the mean-squared error for each coordinate scales as
$\max\{\frac{m p_j}{n}, \frac{\sigma^2 m^2}{n^2}\}$.
For element-level privacy, if we take $t = \log n$ in the definition
of $\rho$, we obtain
mean-squared error scaling as
\begin{equation*}
  \E\left[\left(
    \mechanism_j(\sample, \rho, \{c_k\}) - m p_j\right)^2
    \right]
  \le O(1) \cdot \max_{c \in \{c_k\}}
  \max\left\{\frac{m p_j}{n}, \frac{\sigma^2}{n^2}
  \left[m^2 P(c)^2 + \log^2 m + \log^2 n\right]\right\}.
\end{equation*}
Thus, whenever the individual contribution sizes $m$ are large while
probabilities of elements $P(c)$ are small, element-level mechanisms
allow much more accurate estimation of frequencies than standard
private noise addition.
Of course, the best choice of the projection threshold $\rho$ for
element-level privacy requires some knowledge of the rough probabilities of
each cluster, as otherwise, it is impossible to choose $\rho$ appropriately;
a two-stage estimator (to give rough upper bounds on the element
probabilities $P(c)$) makes this feasible.

%% file: statistical-learning.tex
\subsection{Statistical learning, risk minimization, and M-estimation}
\label{sec:statistical-learning}

\newcommand{\stepsize}{\alpha}
\newcommand{\steppow}{\beta}
\newcommand{\losselement}{\loss_{\mathsf{el}}}
\newcommand{\covlosselement}{\Sigma_{\mathsf{el}}}
\newcommand{\loss}{\ell}
\newcommand{\risk}{L}
\newcommand{\riskelement}{\risk_{\mathsf{el}}}
\newcommand{\proj}{\mathsf{proj}}
\newcommand{\usercard}{M}

Our final application is a fairly careful investigation of statistical
learning problems in the context of element-level differential privacy and
realistic federated learning problems, where individuals contribute more
than a single data point (e.g.\ because they send many text messages).  The
typical statistical learning or generic
M-estimation problem~\cite{HastieTiFr09, VanDerVaart98} is as follows: for
a sample space $\mc{X}$ and parameter space $\Theta$, we have a loss $\loss
: \Theta \times \mc{X} \to \R_+$, where $\loss(\theta; \statval)$ measures
the loss of a parameter $\theta$ on observation $\statval$, and we wish to
minimize the average loss over a population $P$.
In standard empirical risk minimization or M-estimation,
one receives $X\sups{u} \simiid P$, then chooses
$\what{\theta}_n$ to minimize the empirical
average $\frac{1}{n} \sum_{u = 1}^n \loss(\theta; X\sups{u})$.

In our context of element privacy, we modify this slightly. Individuals
(users) contribute batches of data $x \subset \mc{X}$, where the users are
drawn from an underlying population $P$.  Recalling
Section~\ref{sec:element-level}, we assume that there is a prespecified
partition $\{c_1, \ldots, c_K\}$ of $\mc{X}$, so that the element of
protection is whether a user with data $x = \{x_1, \ldots, x_m\}$ has any
individual datum $x_i \in c_k$.  Then the \emph{element-level loss} for a
data batch $x \in 2^{\statdomain}$ averages losses within each element,
\begin{equation}
  \losselement(\theta; \statval)
  \defeq \sum_{k = 1}^K
  \indic{x \cap c_k \neq \emptyset}
  \frac{1}{\card\{x_i \in c_k\}}
  \sum_{x_i \in c_k} \loss(\theta; x_i),
  \label{eqn:element-loss}
\end{equation}
that is, the sum of average losses in the non-empty elements in $x$.
The idea of the
averaging~\eqref{eqn:element-loss} is to make the loss insensitive to
modification of data belonging to any single $c_k$.
For an underlying population distribution $P$, we then wish to solve
the risk minimization problem
\begin{equation}
  \label{eqn:element-risk}
  \minimize_{\theta \in \Theta}
  \riskelement(\theta) \defeq \E[\losselement(\theta; \statrv)]
  = \int \losselement(\theta; \statval) dP(\statval).
\end{equation}
Given a sample $\sample = \{\statrv\sups{u}\}_{u=1}^n \sim P$, we
approximate the risk~\eqref{eqn:element-risk} with
\begin{equation*}
  \riskelement^n(\theta) \defeq \frac{1}{n} \sum_{u = 1}^n
  \losselement(\theta; \statrv\sups{u}),
\end{equation*}
which we attempt to minimize as a proxy for~\eqref{eqn:element-risk}.
To describe our algorithms and their properties, however, we require a brief
digression to provide a general analysis of stochastic approximation
procedures under noise, giving an asymptotic convergence result
that may be interesting independent of its privacy implications.

\subsubsection{A digression to general stochastic optimization}

Consider a generic population risk minimization problem
\begin{equation}
  \label{eqn:pop-risk}
  \minimize_{\theta \in \Theta}
  \risk(\theta) \defeq \E[\loss(\theta; \statrv)]
  = \int \loss(\theta; \statval) dP(\statval),
\end{equation}
where $\loss : \Theta \times \statdomain \to \R$ is a loss.
We have a sample of size $n$ from the
population $P$, and we instead consider applying a stochastic
approximation algorithm on the empirical risk
\begin{equation}
  \label{eqn:generic-empirical-risk}
  \risk_n(\theta) \defeq \frac{1}{n} \sum_{i = 1}^n
  \loss(\theta; \statrv_i)
\end{equation}
for $\statrv_i \simiid P$.  We consider stochastic projected gradient
methods for the problem~\eqref{eqn:generic-empirical-risk}.  In our proofs
in Appendix~\ref{sec:proof-normality}, we generalize this to
\citeauthor{AsiDu19}'s general \textsc{aProx} (approximate proximal point)
family~\cite{AsiDu19}, though its full
treatment somewhat obscures the privacy issues at hand.

In standard applications of stochastic gradient methods~\cite{RobbinsMo51,
  PolyakJu92, NemirovskiJuLaSh09} to the population risk
problem~\eqref{eqn:pop-risk}, one receives an i.i.d.\ sequence $\statrv_k$
and updates
\begin{equation*}
  \theta_{k + 1} =
  \proj_\Theta(\theta_k - \stepsize_k \nabla \loss(\theta_k; \statrv_k)),
\end{equation*}
where $\proj_\Theta(v) = \argmin_{\theta \in \Theta} \{\ltwo{\theta - v}\}$
denotes the Euclidean projection onto $\Theta$.  We consider a
variant of the projected stochastic gradient method as it applies to
triangular arrays, letting the sample size $n$ vary in the stochastic
gradient update applied to the empirical
risk~\eqref{eqn:generic-empirical-risk}. Focusing on the case when the
losses $\loss$ are smooth and convex, we show that as the number of
iterations and the sample size jointly increase, the projected stochastic
gradient method on the empirical risk~\eqref{eqn:generic-empirical-risk} gives
asymptotically normal iterates. To
that end, consider problems indexed by sample size $n$, with a triangular
array of samples $\sample_n \defeq \{\statrv^n_i\}_{i=1}^n$ for $\statrv_i^n
\simiid P$.  Let $\sigma_n \ge 0$ be a fixed variance, and let $Z_i$ be an
i.i.d.\ sequence of random vectors with $\E[Z_i] = 0$ and $\cov(Z_i) =
\Zcov$.  (We allow $\sigma_n > 0$ because we will use the coming iteration
in a private setting, where noise is essential.)  For each $k \in \N$, let
$\randind(k)$ be an index chosen uniformly at random from $\{1, \ldots,
n\}$, and for $k = 1, 2, \ldots,$ and $n \in \N$, consider the noisy
stochastic projected gradient iteration
\begin{equation}
  \label{eqn:sgd-update-sequence}
  \begin{split}
    \gradmap_k^n & \defeq
    \frac{1}{\stepsize_k}
    \left[\theta_k^n - \proj_\Theta\left(\theta_k^n - \stepsize_k
      \nabla \loss(\theta_k^n; \statrv_{\randind(k)}^n)\right)\right] \\
    \theta_{k+1}^n & \defeq
    \theta_k^n - \stepsize_k \left(\gradmap_k^n + \sigma_n
    Z_k \right).
  \end{split}
\end{equation}

Under a few simplifying assumptions on the problem~\eqref{eqn:pop-risk}
reminiscent of the typical classical conditions for
M-estimation~\cite[Ch.~5.3]{VanDerVaart98}, we can prove that the iterates
$\theta_k^n$ enjoy asymptotic optimality properties as $n,k \to
\infty$.
\begin{assumption}
  \label{assumption:make-it-easy}
  The domain $\Theta \subset \R^d$ is compact convex, and
  there exists $\lipobj : \statdomain \to \R_+$
  such that
  $\loss(\cdot; \statval)$ is $\lipobj(\statval)$-Lipschitz over $\Theta$.
  The minimizer
  $\theta\opt \defeq \argmin_{\theta \in \Theta} \risk(\theta)$
  is unique with $\theta\opt \in \interior \Theta$, and $\risk$ is
  $\mc{C}^2$ in a neighborhood of $\theta\opt$, with $\nabla^2
  \risk(\theta\opt) \succ 0$. In addition, there exists an $\epsilon > 0$
  and $\lipgrad, \liphess : \statdomain \to \R_+$ such that $\loss(\cdot;
  \statval)$ has $\lipgrad(\statval)$-Lipschitz gradient and
  $\liphess(\statval)$-Lipschitz Hessian on the neighborhood $\theta\opt +
  \epsilon \ball \subset \interior \Theta$. Finally,
  $\E[\lipconst_a^2(\statrv)] < \infty$ for $a \in \{0, 1, 2\}$.
\end{assumption}

In the projected stochastic gradient
iteration~\eqref{eqn:sgd-update-sequence}, we assume that we run the
algorithm (on random subsamples) for $k = k(n)$ iterations, where the total
$k$ depends on the sample size.  We usually expect that $k = \gamma n$ for
some $\gamma \ge 1$, though in some cases we may wish to take $k/n \to
\infty$.  We also (typically) assume the variance $\sigma_n$, which we add
for privacy, is decreasing.

We have the following theorem, whose proof we provide in
Appendix~\ref{sec:proof-normality}.
\begin{theorem}
  \label{theorem:normality}
  Let Assumption~\ref{assumption:make-it-easy} hold. Define
  $\wb{\theta}^n_k = \frac{1}{k} \sum_{i = 1}^k \theta_i^n$ and
  assume that $\stepsize_k = \stepsize_0 k^{-\steppow}$ for some
  $\steppow \in (\half, 1)$. Define
  $\losscov = \cov(\nabla \loss(\theta\opt; \statrv))$
  and $\Zcov = \cov(Z)$. Assume
  that
  the iteration count $k = k(n)$ satisfies
  $\lim_{n \to \infty} \frac{k(n)}{n} = \gamma > 0$, and that
  $\lim_{n \to \infty} \sigma_n = \sigma \ge 0$.
  Then as $n \to \infty$,
  \begin{equation*}
    \sqrt{n} (\wb{\theta}^n_k - \theta\opt)
    \cd \normal\left(0,
    \nabla^2 \risk(\theta\opt)^{-1}\left(\losscov
    + \frac{1}{\gamma} (\losscov + \sigma^2 \Zcov) \right)
    \nabla^2 \risk(\theta\opt)^{-1}
    \right).
  \end{equation*}
\end{theorem}

We provide a bit of commentary. First, the optimal covariance possible (by
the local asymptotic minimax theorem for stochastic
optimization~\cite{DuchiRu19}) for any estimator of $\theta\opt$ given $n$
observations is $\nabla^2 \risk(\theta\opt)^{-1} \losscov
\nabla^2 \risk(\theta\opt)^{-1}$. Thus, if $\gamma =
\lim\frac{k}{n}$ is large, we have limited asymptotic efficiency loss;
moreover, if the limiting variance $\sigma_n^2 \to \sigma^2$ is zero,
then the efficiency loss is precisely the factor
$1 + 1/\gamma$. In our privacy application, there
is a tradeoff between $k$, the number of iterations,
and the scale $\sigma_n$ of the necessary noise given a sample of size $n$.

\subsubsection{A private stochastic gradient method}

\newcommand{\mechsgd}{\mechanism_{\sigma, \rho, \stepsize, q}}
\newcommand{\mechsgdk}{\mechanism_{\sigma, \rho, \stepsize_k, q}}
\newcommand{\elementupdate}[2]{%
  \mathsf{sgd}\textup{-}\mathsf{el}^\loss_{\stepsize,\rho}({#1}, {#2})
}

We now turn to the appropriate variant of the projected gradient
method~\eqref{eqn:sgd-update-sequence} for privacy. The key from an
element-level privacy perspective is to apply a projected
gradient update on each of a user's elements, then average
them together.  Algorithm~\ref{algorithm:on-device-update}
captures this.

\begin{algorithm}[ht]
  \caption{Element-level projected gradient update
    $\elementupdate{\theta_0}{\statval}$}
  \label{algorithm:on-device-update}
  \begin{algorithmic}
    \Require Projection parameter $\rho$, stepsize $\stepsize$,
    initial model $\theta_0$,
    partition of $\mc{X}$ into
    $\mc{C} = \{c_1, \ldots, c_K\}$, and user data
    $x = \{x_1, \ldots, x_m\}$
    \State \textbf{for each} $k \in \{1, \ldots, K\}$ such that $x \cap c_k
    \neq \emptyset$
    \State ~~~Set $\mc{B} = \{x_i : x_i \in c_k\}$
    \State ~~~$\theta^+_k \leftarrow
    \proj_\Theta(\theta_0 - \stepsize \frac{1}{|\mc{B}|}
    \sum_{x \in \mc{B}} \nabla \loss(\theta_0; x))$
    \State ~~~$ \Delta_k \leftarrow (\theta^+_{k} - \theta_0) / \stepsize$
    ~~~\mbox{and} ~~~$[\Delta_k]_\rho \leftarrow \Delta_k
    \min \{1, \frac{\rho}{\ltwo{\Delta_k}}\}$ 
    \State \Return $\sum_k [\Delta_k]_\rho$
  \end{algorithmic}
\end{algorithm}

Because Algorithm~\ref{algorithm:on-device-update} divides its updates
into the clusters $c_k$ before computing projections (clipping
them to a particular radius) and updates,
its combination with appropriate noise
immediately yields several privacy properties.
The most important result for us is to apply
Alg.~\ref{algorithm:on-device-update}
in a stochastic-gradient-type scheme, which allows us to both leverage
the moments-accountant
(recall Corollary~\ref{corollary:momentacc})
and convergence guarantees of stochastic gradient-type methods.
Following the subsampling~\eqref{eqn:subsampled-sum},
for $q \in (0, 1)$ let
$B_u \simiid \bernoulli(q)$ or $B_u$ be uniform
on $\sum_u B_u = q n$, and
for a sample $\sample = \{x\sups{u}\}_{u=1}^n$ define the subsampled
mechanism
\begin{equation*}
  \mechsgd(\sample; \theta_0)
  \defeq \bigg(\sum_{u = 1}^n
  B_u \cdot \elementupdate{\theta_0}{\statval\sups{u}}\bigg)
  + \normal(0, \rho^2 \sigma^2 I).
\end{equation*}
For any sequence of stepsizes, we may define
the \emph{private stochastic approximation method}
\begin{equation}
  \label{eqn:private-element-sgd}
  \theta_{k+1} \defeq \theta_k - \stepsize_k
  \frac{1}{qn} \mechsgdk(\sample; \theta_k).
\end{equation}
We consider the privacy of the iteration~\eqref{eqn:private-element-sgd}
both in the standard (user-level) private scenario and under element-level
privacy. It is immediate that the update $\elementupdate{\theta_0}{\cdot}$
in Alg.~\ref{algorithm:on-device-update} has element sensitivity at most $2
\rho$, where neighboring data $x, x'$ guarantee
$\ltwos{\elementupdate{\theta_0}{\statval} -
  \elementupdate{\theta_0}{\statval'}} \le 2 \rho$.
For standard privacy, we consider the global sensitivity of the
update: assuming the upper bound
$\card(x) \le \usercard$ on the
cardinality of user data,
we have
$\ltwos{\elementupdate{\theta_0}{\statval}
  - \elementupdate{\theta_0}{\statval'}} \le 2 (K \wedge \usercard)
\rho$ for any two sets $\statval, \statval' \subset \mc{X}$.
We immediately obtain the
following two corollaries on the privacy of the private stochastic gradient
update~\eqref{eqn:private-element-sgd}.

\begin{corollary}
  \label{corollary:all-thetas-are-pretty-private}
  Let $\theta_{1:T} \defeq \{\theta_1, \ldots, \theta_T\}$ be the outputs of
  the iteration~\eqref{eqn:private-element-sgd} and $\diffp_\renparam(q,
  \sigma)$ be as in Corollary~\ref{corollary:momentacc}.  Then
  $\theta_{1:T}$
  is $(T \diffp_\renparam(q, \sigma),
  \renparam)$-element-level R\'{e}nyi private, and for any $\delta > 0$,
  is $(\inf_{\renparam \ge 1} \{T \diffp_\renparam(q, \sigma)
  + \frac{\log \delta^{-1}}{ \renparam - 1}\}, \delta)$-element-level
  differentially private.
\end{corollary}
\begin{corollary}
  \label{corollary:all-thetas-user-level}
  Let the conditions of
  Corollary~\ref{corollary:all-thetas-are-pretty-private} hold. Let
  $\sigma_{\textup{std}} = \sigma / (K \wedge \usercard)$.  Then
  $\theta_{1:T}$ is $(T \diffp_\renparam(q, \sigma_{\textup{std}}),
  \renparam)$-R\'{e}nyi differentially private, and for any $\delta > 0$, is
  $(\inf_{\renparam \ge 1} \{T \diffp_\renparam(q, \sigma_{\textup{std}}) +
  \frac{\log \delta^{-1}}{ \renparam - 1}\}, \delta)$-element-level
  differentially private.
\end{corollary}

Pursuing the discussion following
Corollary~\ref{corollary:momentacc}, let us assume we subsample a constant
fraction $q = m/n$ of the data in the
iteration~\eqref{eqn:private-element-sgd}, where $m$ is fixed and does not
grow with $n$.  Then for $0 < \delta < 1$, the entire
collection $\theta_{1:T}$ is $(O(1) \diffp, \delta)$-element-level
differentially private, where
\begin{equation}
  \diffp \le
  \inf_{\renparam \in [0, \sigma^2 \log \frac{n}{m}]} \left\{
  \frac{T q^2}{\sigma^2}
  +
  \frac{T q^2 \renparam}{\sigma^2}
  + \frac{\log \delta^{-1}}{\renparam} \right\}
  = \frac{T m^2}{n^2 \sigma^2}
  + O(1) \cdot \max\bigg\{\sqrt{\frac{T m^2}{n^2\sigma^2}
  \log \frac{1}{\delta}},
  \frac{\log \delta^{-1}}{\sigma^2 \log \frac{n}{m}} \bigg\}.
  \label{eqn:epsilon-level-sgd}
\end{equation}

\subsubsection{Applications of element-level private stochastic approximation}

While the updates~\eqref{eqn:private-element-sgd}
provide privacy no matter the loss, their utility comes in conjunction with our
analysis in Theorem~\ref{theorem:normality}.
To that end, we now provide a generic convergence result with
a brief application to generalized linear model estimation; our
coming experiments evidence the utility of our definitions and mechanisms.
We first recall the element-level population risk~\eqref{eqn:element-risk},
which averages a standard loss $\loss$ into the element-level
loss $\losselement$. We make a few additional assumptions
on the standard loss $\loss$ over our data $\statdomain$
parallelling Assumption~\ref{assumption:make-it-easy}.
\begin{assumption}
  \label{assumption:allow-privacy}
  There exists $\lipobj < \infty$ such that
  $\theta \mapsto \loss(\theta; \statval)$ is $\lipobj$-Lipschitz
  over $\Theta$ for each $\statval \in \statdomain$.
  The minimizer $\theta\opt \defeq \argmin_{\theta \in \Theta}
  \riskelement(\theta)$ is unique with $\theta\opt \in \interior \Theta$,
  and $\riskelement$ is $\mc{C}^2$ on an $\epsilon$-neighborhood
  of $\theta\opt$ with $\nabla^2 \riskelement(\theta\opt) \succ 0$.
  There are $\lipgrad, \liphess : \statdomain \to \R_+$ such that
  $\loss(\cdot; \statval)$ has $\lipgrad(\statval)$-Lipschitz
  gradient and $\liphess(\statval)$-Lipschitz Hessian
  on $\theta\opt + \epsilon \ball \subset \Theta$,
  where $\E[\lipconst_a^2(\statrv)] < \infty$ for $a \in \{1, 2\}$.
\end{assumption}

The key consequence of the first Lipschitz condition in
Assumption~\ref{assumption:allow-privacy} is that no projection is necessary
in Alg.~\ref{algorithm:on-device-update} to guarantee that
$\elementupdate{\theta}{\cdot}$ has bounded element sensitivity, so that the
private iteration~\eqref{eqn:private-element-sgd} is an instance of the
stochastic approximation iteration~\eqref{eqn:sgd-update-sequence}.
Indeed, the Lipschitz condition is equivalent to $\ltwos{\partial
  \loss(\theta; \statval)} \le \lipobj$ for all $\theta \in \Theta$, and in
turn, the definition~\eqref{eqn:sgd-update-sequence} guarantees that
$\ltwos{\gradmap_k^n} \le \ltwos{\nabla \loss(\theta_k^n;
  \statrv_{\randind(k)}^n)} \le \lipobj$ (cf.~\cite{DuchiRu18c}). As a
consequence, the element-level update of
Algorithm~\ref{algorithm:on-device-update} performs no projection in the
definition $[\Delta_k]_\rho$ whenever $\rho \ge \lipobj$.

Now, recall the asymptotic normality result of
Theorem~\ref{theorem:normality}.  For each $n$ we let
$\{\statrv\sups{u}\}_{u=1}^n \simiid P$ and $\theta_i^n$ be generated by the
iteration~\eqref{eqn:private-element-sgd} for the given sample
$\{\statrv\sups{u}\}_{u=1}^n$ and assume the projection level $\rho \ge
\lipobj$.  Let the stepsizes $\stepsize_k = \stepsize_0 k^{-\steppow}$ for
some $\steppow \in (\half, 1)$ and subsampling rate $q = m / n$ for a fixed
$m$.  Combining Theorem~\ref{theorem:normality} with
Corollaries~\ref{corollary:all-thetas-are-pretty-private}
and~\ref{corollary:all-thetas-user-level} and the discussion immediately
following~\eqref{eqn:epsilon-level-sgd}, we have the following proposition,
which shows that the private stochastic iteration guarantees both asymptotic
normality, and privacy.
\begin{proposition}
  \label{proposition:private-normality}
  Let Assumption~\ref{assumption:allow-privacy} hold, and define
  $\wb{\theta}^n_k = \frac{1}{k} \sum_{i = 1}^k \theta_i^n$, where the
  number of iterations $k = k(n)$ satisfies $\lim_n k(n) / n = \gamma$. Let
  $\covlosselement = \cov(\nabla \losselement(\theta\opt; \statrv))$.  Then
  \begin{equation*}
    \sqrt{n} (\wb{\theta}^n_k - \theta\opt)
    \cd \normal\left(0,
    \nabla^2 \riskelement(\theta\opt)^{-1}
    \left(\covlosselement + \frac{1}{\gamma}\Big(\frac{1}{m}
    \covlosselement
    + \frac{\rho^2 \sigma^2}{m^2} I\Big)\right)
    \nabla^2 \riskelement(\theta\opt)^{-1}
    \right).
  \end{equation*}
  Fix $\delta > 0$ and let $\diffp(\tau) = \inf_{\renparam} \{ \frac{\gamma
    m^2}{n \tau^2} + \frac{\gamma m^2}{n \tau^2} \renparam +
  \frac{\log\delta^{-1}}{\renparam} \mid \renparam \le \tau^2 \log
  \frac{n}{m}\}$ for shorthand. Then
  \begin{enumerate}[label=(\roman*)]
  \item If $\sigma^2 \ge 2$, then the
    collection $\{\theta_i^n\}_{i = 1}^k$ is
    $(O(1) \cdot \diffp(\sigma), \delta)$-element-level differentially private.
  \item Assume that each user data $\statval$ has cardinality at most
    $\card(\statval) \le \usercard$. If $\sigma^2 \ge
    (K \wedge \usercard)^2 \tau^2$, where $\tau^2 \ge 2$,
    then
    then the
    collection $\{\theta_i^n\}_{i = 1}^k$ is
    $(O(1) \cdot \diffp(\tau), \delta)$-differentially private.
  \end{enumerate}
\end{proposition}
\noindent
As in the preceding examples, we see roughly the same tradeoffs between
user-level (standard) and element-level privacy: for a given
level $\diffp$, it is possible to provide
the less-stringent element-level privacy
with noise a factor
$K \wedge \usercard$ less than that for user-level privacy.

In general, the partitioning that the element-level
loss~\eqref{eqn:element-loss} and risk~\eqref{eqn:element-risk} in the data
space may change the resulting estimated parameters from more standard
sampling schemes. However, any normalization of user's data (as some users
contribute many data points, some contribute few) in any application
engenders changes in the ``optimal'' parameter $\theta\opt$, so we believe
this of limited impact.
To give a somewhat concrete example, consider generalized linear
models (GLMs)~\cite{McCullaghNe89}:

\begin{example}[Generalized linear models]
  In a GLM, for an individual data point $x \in \R^d$
  we have $Y$ with density (or p.m.f.)
  \begin{equation*}
    p_\theta(y \mid x) = \exp(T(y) \theta^T x - A(\theta; x)) h(y),
  \end{equation*}
  where $h$ is a base measure, $A(\theta; x)$ is the log-partition function
  $A(\theta; x) = \log \int e^{T(y) \theta^T x} h(y) dy$, and $T$ the
  sufficient statistic. In this case
  for loss $\loss(\theta; x, y) = -\log p_\theta(y \mid x)$,
  any partition of $\mc{X}$ into elements guarantees that
  $\theta\opt = \argmin_\theta \riskelement(\theta)$ remains fixed.
  The Fisher information may change, of course:
  given a partition of $\mc{X}$ into clusters $\{c_k\}_{k = 1}^K$, defining
  $p_k = \P(\statrv\sups{u} \cap c_k \neq \emptyset)$, we have
  $\riskelement(\theta) = \sum_{k = 1}^K p_k \E[\loss(\theta; X, Y) \mid X
    \in c_k]$, so that
  modifying the partition $c_k$ changes $\nabla^2 \riskelement(\theta\opt)$
  and $\covlosselement$; in some situations, this can decrease the
  asymptotic variance, while in others, it may increase, depending
  on the degree of stratification and relative probabilities.
\end{example}

%% file: experiments.tex
\section{Experiments}
\label{sec:experiments}

To demonstrate the behavior of element-level private mechanisms, we present
a series of experimental results in crowdsourced (federated) learning and
stochastic optimization. We perform both simulations, where we may control
all aspects of the experiments, and real-world experiments.
Our theoretical results and intuition suggest that as
the number of elements we consider grows---meaning that
the elements provide a finer partition of the input space $\mc{X}$---we
should observe performance improvements.
%
%
In large-scale estimation, such as federated
learning~\cite{McMahanMoRaHaAr17}, user data is rarely i.i.d.  For example,
some users take many photos of their children, others of dogs, others of
hikes with friends; thus, a user may provide data only relating to a few
elements. Motivated by this potential variability, for
datasets with no pre-existing users, we diversify our
experiments by constructing pseudo-users and assigning them varying
numbers of elements.

In the remainder of the section, we present results for histogram estimation
(Sec.~\ref{sec:histogram-experiment}), a simulated logistic regression
experiment (Sec.~\ref{sec:simulated-logreg}), and then two experiments on
fitting large image classification models, the first on tuning a model to a
new dataset based on Flickr images (Sec.~\ref{sec:flickr-experiment}), and
the second an investigation on training a full neural network
(Sec.~\ref{sec:cifar-experiment}).  An essential part of each experiment is
to describe how we choose the elements to protect---this decision is more of
a policy decision than a purely mathematical one, and consequently deserves
care and thought, especially in real-world applications.  In each
experiment, we provide user-level or element-level $(\diffp,
\delta)$-differential privacy, where $\delta = n^{-1.1}$, where $n$ is the
total number of users.



\subsection{Histogram estimation}
\label{sec:histogram-experiment}

We consider the problem of estimating frequent words on a dataset consisting
of Reddit comments~\cite{Baumgartner17}, where unique usernames
identify users. Proposition~\ref{proposition:element-level-histogram}
predicts that element-level privacy with appropriate parameter
settings in the mechanism~\eqref{eqn:histogram-mechanism} should reduce
squared error by a factor of roughly $\max_{c \in \{c_k\}} P(c)^2$, so
that increasing cluster counts should yield further improvements.

\begin{figure}[t]
  \begin{center}
    \vspace{-1cm}
    \begin{overpic}[width=0.9\columnwidth]{
	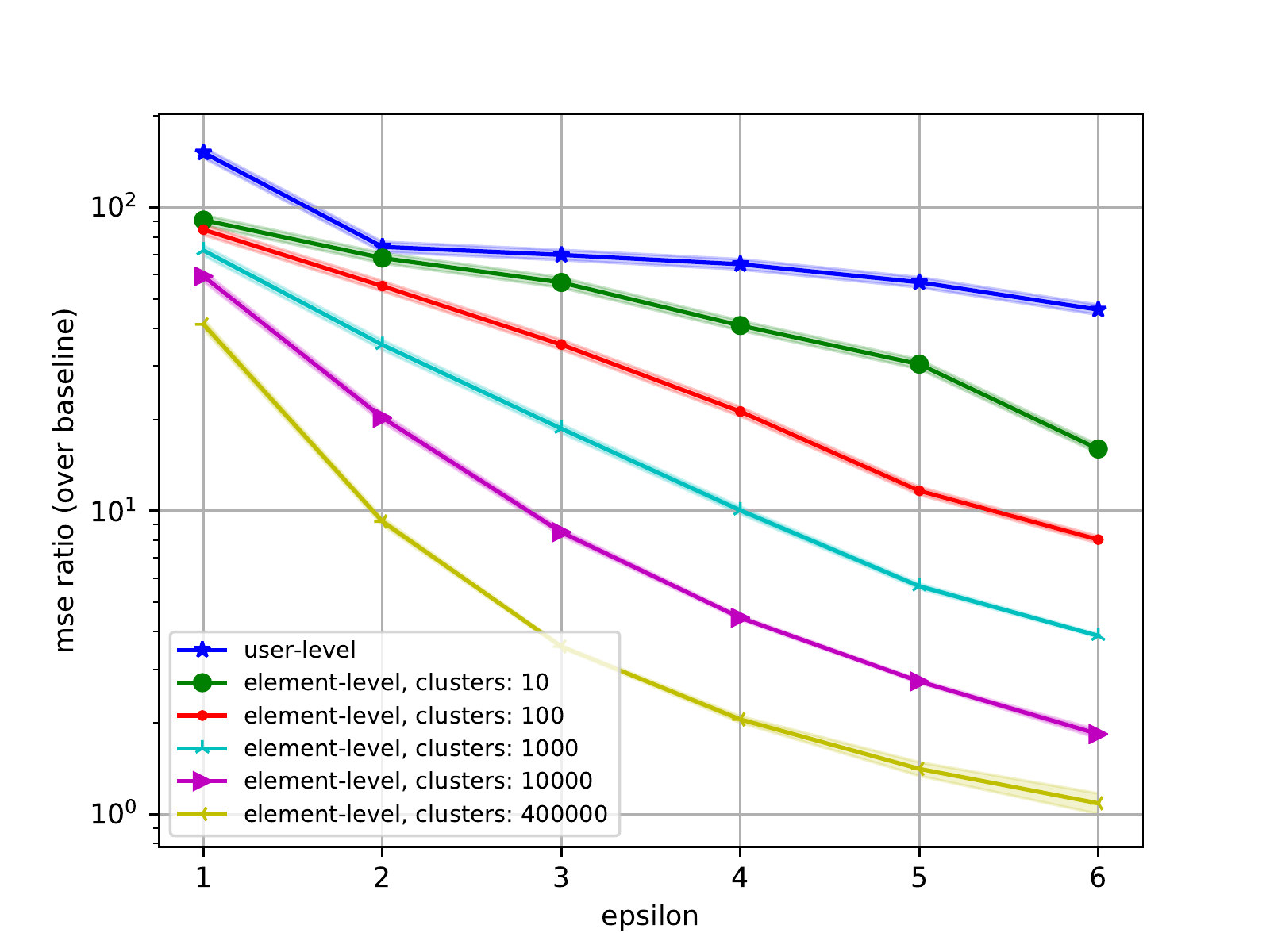}
      \put(2,22){
	\tikz{\path[draw=white,fill=white] (0, 0) rectangle (.6cm,6cm);}
      }
      \put(1,15){
        \rotatebox{90}{Error ratio
          $\ltwobig{\what{H} - H_0}^2 / \ltwo{H_1 - H_0}^2$
        }
      }
      \put(30,1){
	\tikz{\path[draw=white,fill=white] (0, 0) rectangle (4cm,.4cm);}
      }
      \put(43,2){Privacy level $\diffp$}
    \end{overpic} 
    \caption{\label{fig:exp_reddit} Mean-squared error ratio
      $\ltwos{\what{H} - H_0}^2 / \ltwos{H_1 - H_0}^2$
      of private error over baseline error
      for word frequency estimation on a dataset of
      Reddit comments, plotted versus
      privacy level $\diffp$. Confidence intervals are $\pm 1.64$
      standard errors.
      Each line corresponds to a partition of the space of
      words into $1$ (user-level) or more clusters
      $c_k$.}
  \end{center}
\end{figure}
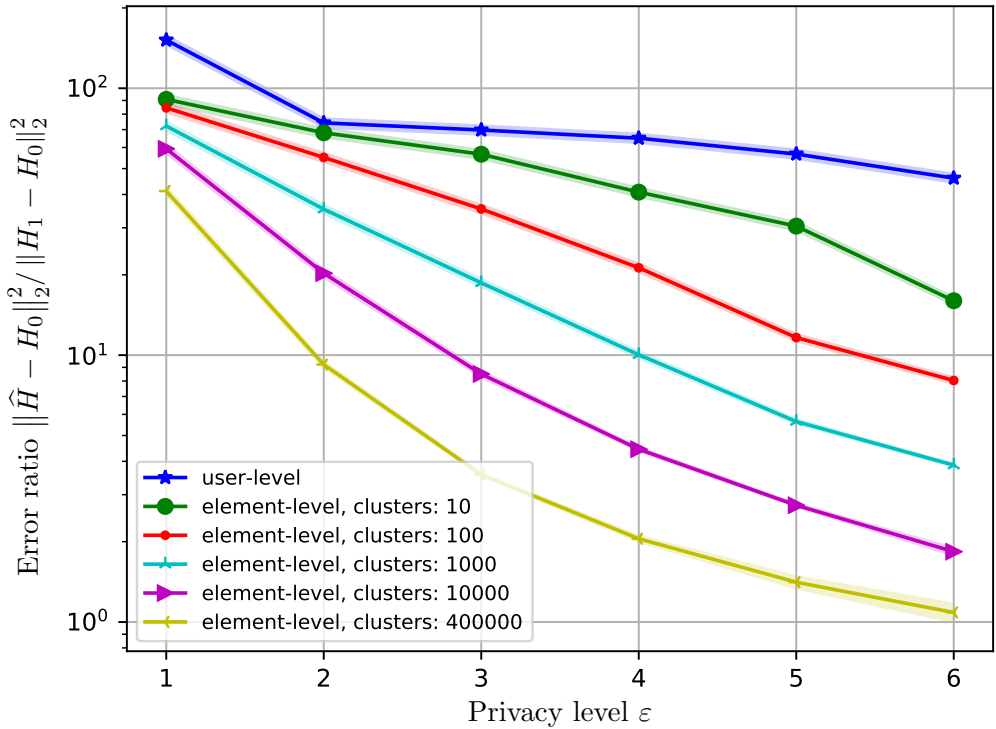

In the experiment, we consider the first $n = 2000$ users with the largest
number of tokens (words), using as our dictionary those words in the
vocabulary of GloVe (Global Vectors for Word
Representation)~\cite{PenningtonSoMa14}, yielding dictionary of size $d =
400000$, where we choose a random subsample of each user's words to obtain $m
= 4000$ words per individual.  Additionally, we remove the 100 most
frequent stopwords (e.g.\ ``the'', ``and'', ``a'').  GloVe embeds words into
$\R^{100}$, and using these embedded vectors, we cluster the $d = 4 \cdot
10^5$-sized vocabulary into $K = 10, 100, 1000, 10000,$ and $400000$
clusters (elements); assuming the embedding is ``semantically meaningful''
as claimed~\cite{PenningtonSoMa14}, these elements should naturally
demarcate themes and conversation foci.  Within each experiment, we
calculate the histogram to be estimated by first randomly dividing users
into two disjoint sets $\sample_0$ and $\sample_1$ and defining the ``true''
histogram $H_0 = \frac{1}{m|\sample_0|}\sum_{u \in \sample_0} X\sups{u}$.
We then estimate $H_0$ using the sample $\sample_1$ via the
mechanism~\eqref{eqn:histogram-mechanism},
$\what{H}=\frac{1}{m}\mechanism(S_1, \rho, \{c_k\})$.
In each individual experiment---that is, for each choice of privacy level
$\diffp$ and total number of clusters---we use a validation
set to choose the truncation threshold
$\rho \in \{1, 2, \ldots, 10\} \cup \{15, 20, \ldots, 50\}
\cup \{70, 100, 150, 200\}$ minimizing the mean-squared error, so
that our results reflect the best behavior for each method.

We estimate the baseline mean squared error to be $\ltwos{H_0 - H_1}^2$.  In
Figure \ref{fig:exp_reddit}, we plot the ratio
$\ltwos{\what{H} - H_0}^2 / \ltwos{H_1 - H_0}^2$ of squared
error for the private estimation algorithm over the
baseline mean squared error against the privacy parameter $\diffp$.  The
results broadly are as expected: increasingly fine partitions yield better
estimators. Moreover, for a given privacy level $\diffp$, the separation
between the mean-squared error is roughly linear on a logarithmic scale,
which is what we expect from reductions scaling as $\max_{c \in \{c_k\}}
P(c)^2$.


\subsection{Simulated logistic regression}
\label{sec:simulated-logreg}


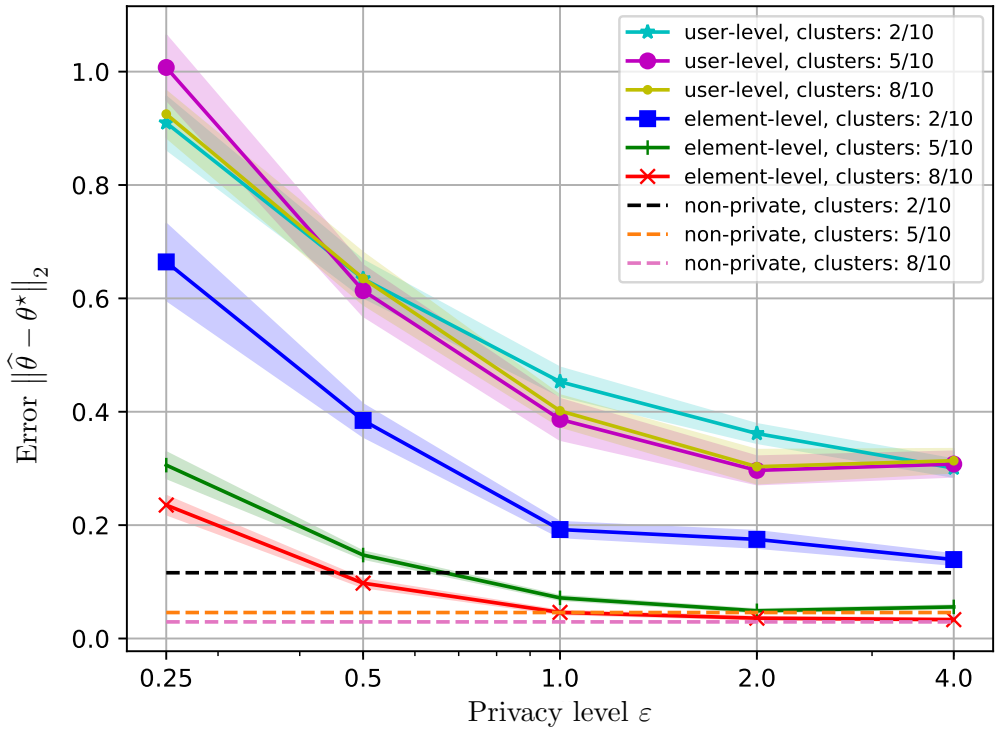
\begin{figure}[t]
  \begin{center}
    \vspace{-1cm}
    \begin{overpic}[width=0.9\columnwidth]{
	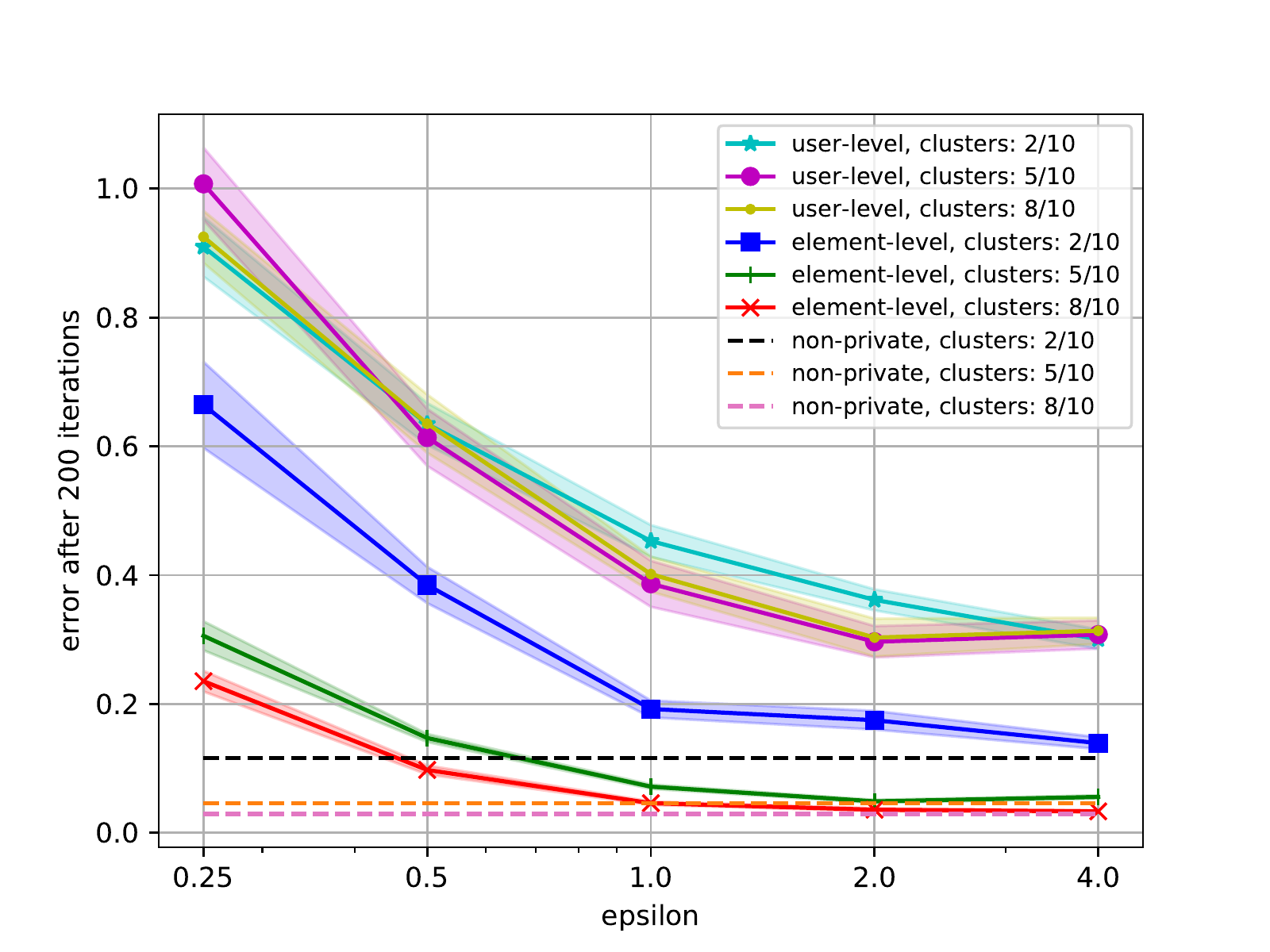}
      \put(2,22){
	\tikz{\path[draw=white,fill=white] (0, 0) rectangle (.6cm,6cm);}
      }
      \put(1,25){
        \rotatebox{90}{Error
          $\ltwobig{\what{\theta} - \theta\opt}$
        }
      }
      \put(30,1){
	\tikz{\path[draw=white,fill=white] (0, 0) rectangle (4cm,.4cm);}
      }
      \put(43,2){Privacy level $\diffp$}
    \end{overpic} 
    \caption{\label{fig:logistic-simulation} Logistic regression simulation
      with user sample size $n=1000$, each user providing
      $m = 50$ datapoints, in dimension $d = 10$. The horizontal axis
      indexes privacy parameter $\diffp$ while the vertical axis indexes
      error $\ltwos{\what{\theta} - \theta^\star}$
      after $T=200$ stochastic gradient updates~\eqref{eqn:private-element-sgd}.
      Confidence intervals are $\pm 1.64$ standard errors.
      The legend label
      ``clusters: $k / 10$'' represent the number of clusters (elements)
      a user has ($k_u$) over the total possible
      number of distinct clusters ($10$).
    }
  \end{center}
\end{figure}

The remainder of our experiments consider M-estimation and statistical risk
minimization, as in Section~\ref{sec:statistical-learning}.  We begin with a
simulation study to more precisely control the hypotheses and experiments,
focusing on logistic regression.  For each experiment, we generate data via
the following hierarchical model: first, we draw $K = 10$ element centers
$c_1, \ldots, c_K \simiid \uniform(\sphere^{d-1})$ and $\theta\opt \sim
\uniformdist(\sphere^{d-1})$. Then we generate pairs $(X_i, Y_i) \in \R^d
\times \{\pm 1\}$ according to the logistic model
\begin{equation*}
  p_\theta(y \mid x) = \frac{1}{1 + \exp(-y \<x, \theta\opt\>)},
  ~~~
  C_i \sim \uniform(\{c_k\}_{k=1}^K)
  ~~ \mbox{and} ~~
  X_i = C_i + \uniform(\sphere^{d-1}),
\end{equation*}
so each datum $X_i$ belongs to the cluster around element
$C_i$. Following the model that users provide several data points,
we generate data for $n = 1000$ users, each consisting of
$m = 50$ pairs $(x_i, y_i)$.

Given a collection of users, we apply the private
stochastic gradient method~\eqref{eqn:private-element-sgd}
with the element-level update $\elementupdate{\theta}{(\statval, y)}$
of Alg.~\ref{algorithm:on-device-update}. We vary the diversity
of data users provide, so that in different experiments users
provide data from $k = 2, 5, 8$ of the $K = 10$ clusters; we expect that
the more diverse the data the users provide (i.e.\ coverage of clusters),
the more element-level privacy should improve over standard (user-level)
private mechanisms.

We calculate the privacy parameter $\diffp$ for both user- and element-level
privacy using \citeauthor{AbadiChGoMcMiTaZh16}'s moments accountant
(Corollary~\ref{corollary:momentacc} and~\cite{AbadiChGoMcMiTaZh16}).  We
perform $T = 200$ private updates~\eqref{eqn:private-element-sgd}, choosing
stepsize $\stepsize_k = \stepsize_0 / \sqrt{k}$. In any real-world
deployment, one chooses hyperparameters to maximize a method's performance,
so for each fixed privacy level $\diffp$, we (experimentally) find a
subsampling rate $q$ and initial stepsize $\stepsize_0$ to yield the best
performance for each method.

We show results in Figure~\ref{fig:logistic-simulation}, where we plot the
error $\ltwos{\what{\theta} - \theta\opt}$ for the final estimated
$\what{\theta}$ of the private stochastic gradient iteration against the
provided privacy level $\diffp$.  Broadly, the results are as expected: as
we increase the diversity of elements for which each user has data, the
estimation error decreases for a given element privacy level $\diffp$, while
user-level private mechanisms exhibit little change on this axis. Of note,
however, is the baseline error: the more clusters (i.e.\ more stratified the
data per user), the better a \emph{non-private} stochastic gradient scheme
estimates $\theta\opt$. We believe this occurs because the stratification
of data within users improves problem conditioning. Even
with this difference, however, the convergence of the error of the
private stochastic gradient method to that of the non-private
error is faster for scenarios with more clusters.




\subsection{Large-scale multiclass image classification: the Flickr dataset}
\label{sec:flickr-experiment}

Following our simulated logistic regression results, we
investigate element-level privacy in the context of
model fitting for a large image
classification task, following our methodology in
Section~\ref{sec:statistical-learning}. In this experiment, we vary several
parameters: the privacy level $\diffp \in \{1, 3, \infty\}$, the number of
distinct clusters into which we partition the input space ($K = 50, 500$),
and, as we discuss in the introduction to the experiments, we also vary the
diversity of images of individual users, so that we provide nominal
``users'' with data from 5, 30, or 100 distinct clusters/elements. As
in the previous experiments,
we expect the following: as the number of clusters $K$ increases,
element-level private methods should improve relative to the user-level
private method, and similarly, as the diversity of individual
users' images increases (the number of distinct elements), we expect to see
further relative improvement.  This is
natural: in Algorithm~\ref{algorithm:on-device-update} and
the update~\eqref{eqn:private-element-sgd}, the magnitude of
noise addition relative to the scale of a user's contribution decreases
linearly in the number of distinct elements a user provides.

To this end, we perform a model tuning experiment on the Flickr dataset
\cite{ThomeeShFrElNiPoBoLi16} using a ResNet50 network \cite{HeZhReSu16}
pre-trained on ImageNet~\cite{DengDoSoLiLiFe09}, with reference
implementation~\cite{PaszkeGrChChYaDeLiDeAnLe17}.  This tuning means we fit
only the \emph{last layer} of the network, that is, we fit a multiclass
logistic regression on input features $x \in \R^d$, $d = 2048$, defined by
the outputs of the second-to-last ResNet50 layer. We use the 100 most
popular Flickr image tags as labels, which represent 89\% of the chosen
data, and used an ``unknown'' label for anything remaining, resulting in a
101 class multiclass problem. To construct the elements into which we
partition the images, we chose a uniformly random subset of 100,000 Flickr
images, then used KMeans++~\cite{ArthurVa07} to cluster them into $K = 50$
and $500$ clusters. Then a given image representation $x$ simply belongs to
the nearest cluster centroid.  To fit the resulting model, we use the
stochastic gradient method in Algorithm~\ref{algorithm:on-device-update} as
applied in the update~\eqref{eqn:private-element-sgd}.  We construct a
nominal collection of $n=8000$ users, assigning each $m=100$ labeled images
$(x, y)$.  We vary the image allocations, so that (depending on the
experiment) each user has images from on average $k = 5, 30, 100$ distinct
elements.  We perform $T =$ 40,000 updates~\eqref{eqn:private-element-sgd}
in each experiment.

We present results in Figure~\ref{fig:flickr}, plotting the maximum top-5
accuracy achieved (i.e.\ there is no loss if the correct label belongs to
the five highest-scoring predicted labels for an example $x$) versus
iteration for many parameter settings. In the figure, we simultaneously
present results for different privacy levels $\diffp$, number $K$ of
clusters, and diversity of clusters per user. We highlight a few of the most
salient points. First, user-level privacy with $\diffp = 1$ is substantially
worse than any other method. Second, we see roughly what we expect, in that
the element-level private algorithms achieve higher accuracy as the number
of clusters and per-user diversity increase.  Given that true internet-scale
datasets are several times larger than the 400,000 image dataset we
construct, this suggests the element-level private mechanisms can
provide strong utility with satisfactory privacy.

 
\begin{figure}[ht]
  \begin{center}
      \begin{overpic}[width=0.9\columnwidth]{
	  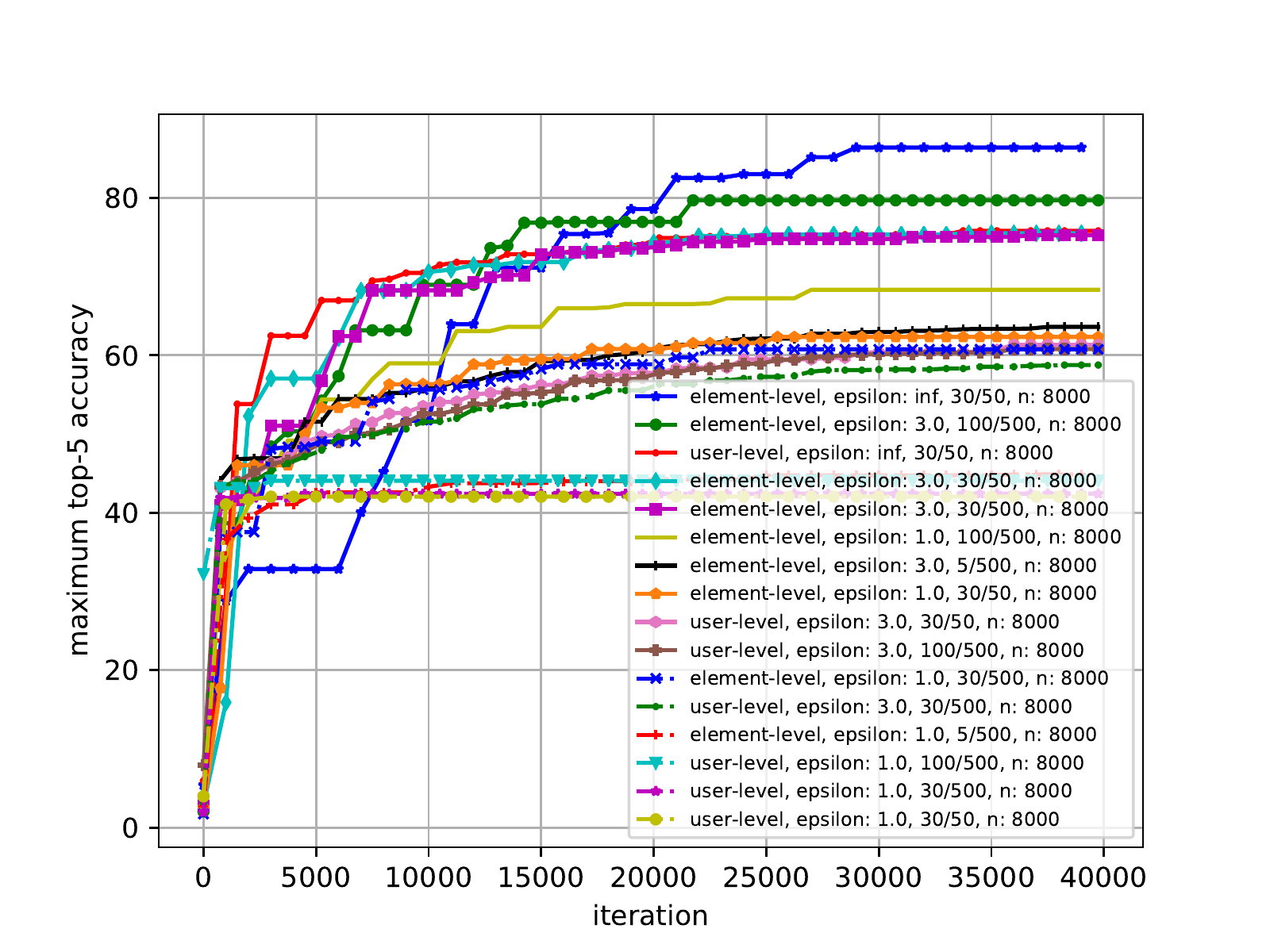}
        \put(53.5,9.5){
          \tikz{\path[draw=white,fill=white] (0, 0) rectangle(5.1cm,5.2cm);}
        }
        \put(53.7,43.2){\footnotesize non-private, 100/500 clusters present}
        \put(53.7,41){\footnotesize element-level, $\diffp = 3.0$,
          100/500}
        \put(53.7,38.9){\footnotesize non-private, 30/50 clusters present}
        \put(53.7,36.7){\footnotesize element-level, $\diffp = 3.0$, 30/50}
        \put(53.7,34.4){\footnotesize element-level, $\diffp = 3.0$, 30/500}
        \put(53.7,32.2){\footnotesize element-level, $\diffp = 1.0$, 100/500}
        \put(53.7,30){\footnotesize element-level, $\diffp = 3.0$, 5/500}
        \put(53.7,27.75){\footnotesize element-level, $\diffp = 1.0$, 30/50}
        \put(53.7,25.5){\footnotesize user-level, $\diffp = 3.0$, 30/50}
        \put(53.7,23.25){\footnotesize user-level, $\diffp = 3.0$, 100/500}

        \put(53.7,21.0){\footnotesize element-level, $\diffp = 1.0$, 30/500}
        \put(53.7,18.8){\footnotesize user-level, $\diffp = 3.0$, 30/500}
        \put(53.7,16.6){\footnotesize element-level, $\diffp = 1.0$, 5/500}
        \put(53.7,14.4){\footnotesize user-level, $\diffp = 1.0$, 100/500}
        \put(53.7,12.2){\footnotesize user-level, $\diffp = 1.0$, 30/500}
        \put(53.7,10.0){\footnotesize user-level, $\diffp = 1.0$, 30/50}
      \end{overpic} 
      \caption{\label{fig:flickr} Training curves for the private
        element-level stochastic gradient method,
        Alg.~\ref{algorithm:on-device-update} coupled with
        update~\eqref{eqn:private-element-sgd}, over $T =$ 40,000 updates on
        the Flickr dataset. Each line displays the best top-5 prediction
        accuracy achieved before iteration $t$. The legend ratio
        $k / K$ represent the number of clusters (elements) a user has
        $(k_u)$ over the total possible number of distinct clusters $(K)$. }
  \end{center}
\end{figure}
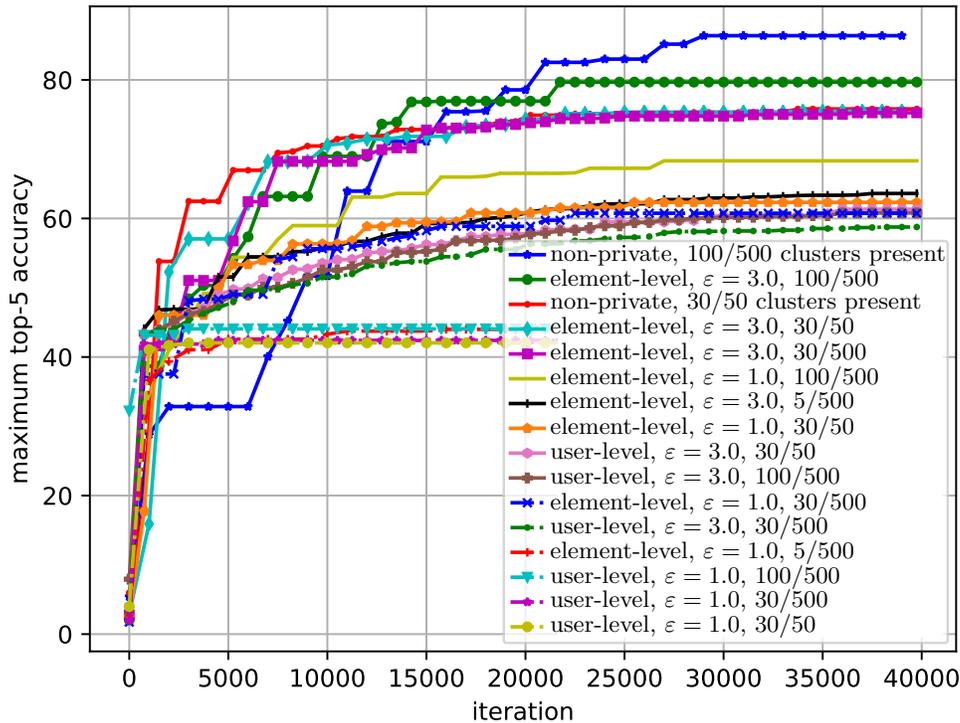

\subsection{Fully training a neural network: image classification on CIFAR10}
\label{sec:cifar-experiment}

We present our final experimental results for a classification problem on
the CIFAR10 dataset~\cite{KrizhevskyHi09}, showing that it is possible to
privately train a neural network while providing element-level privacy. We
use the relatively simple convolutional neural network model architecture in
the PyTorch tutorial~\cite{PaszkeGrChChYaDeLiDeAnLe17}.  To construct the
cluster centroids (elements), we mimic the method we propose for Flickr:
we upsample the CIFAR image (using PyTorch), pass the resulting image through
the pre-trained ResNet50 network above, and then cluster the resulting
2048-dimensional vectors using KMeans++~\cite{ArthurVa07} to
construct $K = 100$ centroids that partition the CIFAR dataset.




We again perform a federated learning experiment over $T = $ 40,000 steps
(Alg.~\ref{algorithm:on-device-update} and
update~\eqref{eqn:private-element-sgd}). Similar to our experiment with
Flickr---except that we train a full neural network---we considered $n \in
\{2, 8\} \cdot 10^3$ users, each assigned $m = 100$ images from $k = 5$ or
$30$ of the $K = 100$ elements we cluster. Users may have repeat data.  In
Figure~\ref{fig:cifar10}, we plot the difference in top-1 accuracy between a
private method and the fully-trained (non-private) tutorial convolutional
neural network~\cite{PaszkeGrChChYaDeLiDeAnLe17} against iteration, varying
the privacy parameter $\diffp$ and cluster diversity.
As expected, we see two effects: first, as the sample size $n$ grows, the
accuracy improves; second, as the diversity of elements per user decreases,
performance degrades as expected. All user-level private instantiations have
accuracy more at least 15\%-worse than the non-private accuracy.
Conversely, the element-level-private algorithm with $\diffp = 3$, $n =
8000$, and high element diversity per-user (30/100 data clusters present)
achieves top-1 accuracy nearly equal to non-private training.


\begin{figure}[ht]
  \begin{center}
      \begin{overpic}[width=0.9\columnwidth]{
	  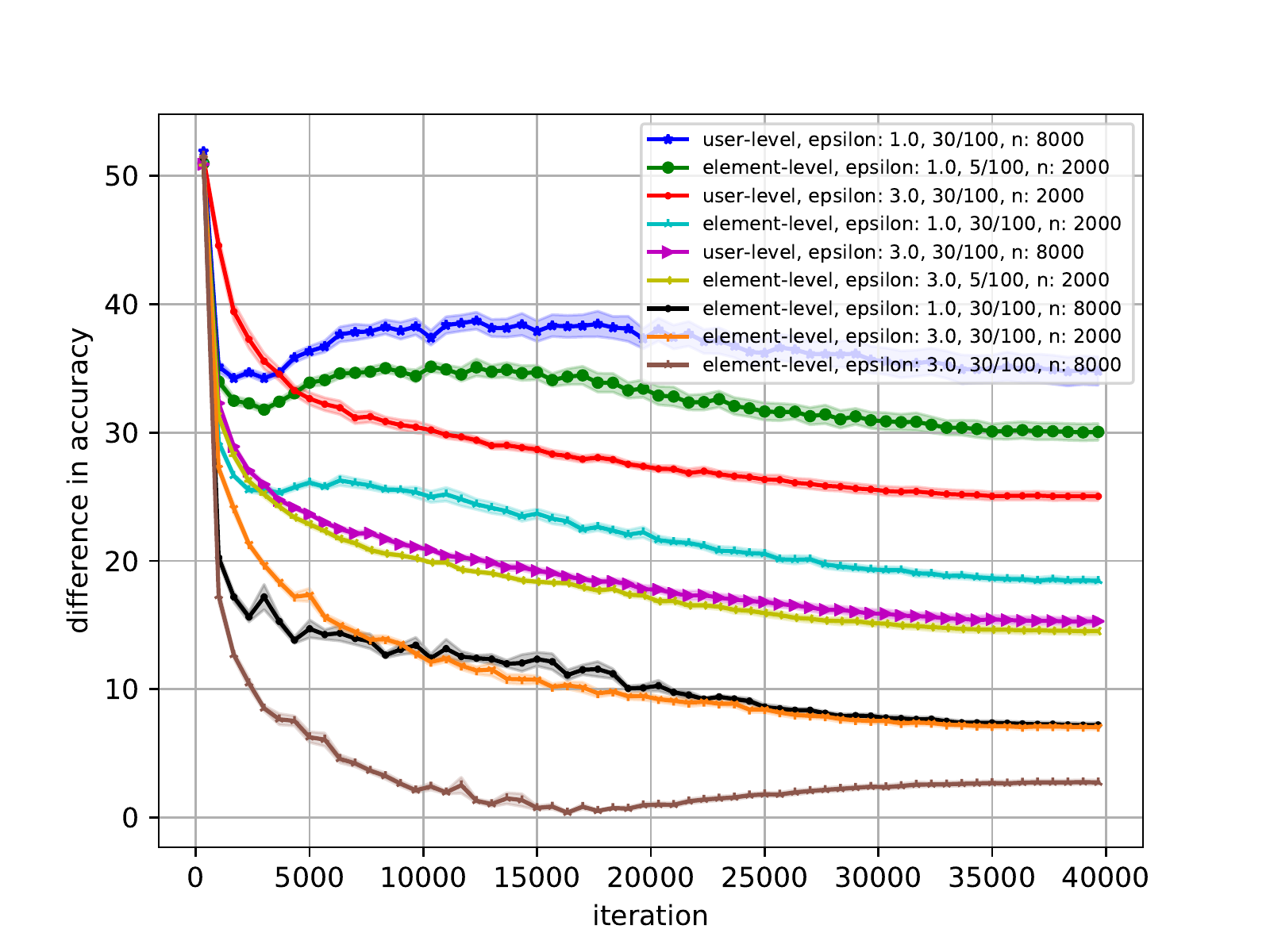}
      \end{overpic} 
      \caption{\label{fig:cifar10} Difference in accuracy of a convolutional
        neural network model on the CIFAR10 dataset trained with privacy and
        without. Each line corresponds to given privacy type (element- or
        user-level), privacy level $\diffp$, effective number of users $n
        \in \{2000,8000\}$, and diversity of elements each user provides ($k
        = 5$ (5/100) or $k = 30$ (30/100)).  Confidence interval are $\pm
        1.64$ standard errors.}
  \end{center}
\end{figure}

%% file: discussion.tex
\section{Discussion}

We conclude with a brief discussion. Element-level differential privacy
allows one to choose the granularity at which to provide privacy
protections. As we see both in the theoretical results and experiments, this
can allow substantially improved utility over standard private
algorithms. This additional flexibility, however, comes with a challenge:
one must carefully choose the elements (partition of the data space
$\mc{X}$) to provide sufficient privacy, as increasing the number of
clusters allows improved statistical accuracy while decreasing the number
improves privacy. This apparent tradeoff requires a per-application policy
decision, where one balances privacy---in the coarseness of the partitioning
into elements---against utility; as in standard privacy, where the choice of
$\diffp$ is a policy decision that must trade privacy against utility, care
is likely necessary here.

%% file: quick-and-dirty.tex

\section{Sufficiency of mechanism~\eqref{eqn:mironov-gaussian-mech}}
\label{sec:sufficiency-of-mironov}

The R\'{e}nyi divergence between Gaussian distributions
$P_i = \normal(\mu_i, \sigma^2 I)$ is $\drenyi{P_0}{P_1}
= \frac{\renparam \norm{\mu_0 - \mu_1}^2}{2 \sigma^2}$. Thus
for any mechanism defined by
$\mechanism(\sample) = f(\sample) + \normal(0, \rho^2 \sigma^2)$, we have
$\drenyi{\mechanism(\sample)}{\mechanism(\sample')}
\le \frac{\renparam}{2 \sigma^2}$, so that
Mironov's results~\cite{Mironov17} guarantee the mechanism is
$(\frac{\renparam}{2 \sigma^2} + \frac{\log(1/\delta)}{\renparam - 1},
\delta)$-differentially private. Setting $\renparam =
1 + \frac{2 \log(1/\delta)}{\diffp}$ and
$\sigma^2 = (1 + 2 \diffp^{-1} \log \frac{1}{\delta}) / \diffp$
gives the result.

%% file: proof-heavy-hitters.tex
\section{Proof of Proposition~\ref{proposition:heavyhitters}}
\label{proof:heavyhitters}

We begin by presenting two lemmas that give us 
the tools to prove the proposition.

\begin{lemma}
  \label{lemma:indicator-properties}
  Let $q_j = 1 - (1 - p_j)^m$, $q_{jl} = 1 - (1 - p_l - p_j)^m$,
  $X \sim \multinomial(m, p)$, and
  $Y = \ones(X)$. Then
  $\E[Y_j] = q_j$, $\E[Y_j^2] = q_j$,
  $\E[Y_jY_l] = q_j + q_l - q_{jl}$, and
  $\var(Y_j - Y_l) \le q_j + q_l$.
\end{lemma}
\begin{proof}
  The first and second claims are immediate. For the third, we have
  \begin{align*}
    \E[Y_j Y_l] & =
    \P(X_j > 0, X_l > 0) \\
    & = 1 - \P(X_j = 0, X_l = 0) - \P(X_j > 0, X_l = 0)
    - \P(X_j = 0, X_l > 0).
  \end{align*}
  As $X_j \mid (X_l = 0)\sim \binomial(m, \frac{p_j}{1 - p_l})$,
  we have
  \begin{align*}
    \P(X_j > 0, X_l = 0)
    & = \P(X_j > 0 \mid X_l = 0) \P(X_l = 0) \\
    & = \left[1 - \left(1 - \frac{p_j}{1 - p_l}\right)^m\right]
    (1 - p_l)^m
    = q_{jl} - q_l,
  \end{align*}
  and we similarly obtain that $\P(X_j = 0, X_l = 0) = 1 - q_{lj}$.
  Algebraic manipulations then give $\E[Y_j Y_l] = q_j + q_l - q_{jl}$.

  Finally, $\var(Y_j - Y_l) \le \E[Y_j^2 + Y_l^2] \le q_j + q_l$, as desired.
\end{proof}
Next, we prove the upper-bound on the probability that the
private mechanism $\what{H}$ mis-orders the two bins $i$ and $j > i$.

\begin{lemma}
  \label{lemma:order-error-bound}
  Let $i < j$, so that $q_i = 1 - (1 - p_i)^m \ge q_j = 1 - (1 - p_j)^m$.
  Then
  \begin{equation*}
    \P\left(\what{H}_i < \what{H}_j\right)
    \le
    \exp\left(-\min\left\{\frac{5 n(q_i - q_j)^2}{12
      v_{ij} + 20 \frac{\sigma^2}{n}},
    \frac{n (q_i - q_j)}{8}\right\}\right).
  \end{equation*}
\end{lemma}
\noindent See Section~\ref{proof:ordererrorbound} for a proof.

By a binomial expansion, we see that if
$p_i + p_j \le \frac{1}{m}$, then
\begin{align*}
  q_i - q_j & =
  (1 - p_j)^m - (1 - p_i)^m \ge m(p_i - p_j)
  - \frac{m^2}{2} (p_i^2 - p_j^2)
  \ge \frac{m}{2} (p_i - p_j) \\
  q_i + q_j
  & = 2 - (1 - p_j)^m - (1 - p_i)^m
  \le m (p_i + p_j) + \frac{m^2}{2} (p_i^2 + p_j^2)
  \le 3 m p_i,
\end{align*}
so Lemma~\ref{lemma:order-error-bound} implies
\begin{align*}
  \E[\indic{p_i - p_j \ge \gamma} \indics{\what{H}_j > \what{H}_i}]
  & \le
  \max\left\{\exp\left(- \frac{5 n m^2 (p_i - p_j)^2}{
    48 v_{ij} + 80 \frac{\sigma^2}{n}}\right),
  \exp\left(-\frac{n m (p_i - p_j)}{16}\right)\right\} \\
  & \le \max\left\{
  \exp\left(-\frac{5 n m^2 \gamma^2}{144 p_i}\right),
  \exp\left(-\frac{n^2 m^2 \gamma^2}{16 \sigma^2} \right),
  \exp\left(-\frac{nm \gamma}{16}\right)\right\},
\end{align*}
where the second inequality used the variance bound $v_{ij} \le q_i + q_j
\le 3 m p_i$ of Lemma~\ref{lemma:indicator-properties}.

Setting $\gamma$ as in the statement of the proposition and summing over all
$i < j$ in the loss $\heavyloss$ of Eq.~\eqref{eqn:heavyhittersloss} gives
the result.


\subsection{Proof of Lemma~\ref{lemma:order-error-bound}}
\label{proof:ordererrorbound}

Let $i < j$ so that $q_i \ge q_j$, as $p_i \ge p_j$ by assumption. Define the
zero-mean random variable
$\Delta_{ij}\sups{u} = Y_i\sups{u} - Y_j\sups{u}
- (q_i - q_j)$. Then we have
\begin{equation*}
  \P(\what{H}_i \le \what{H}_j)
  = \P\bigg(\sum_{u = 1}^n \Delta_{ij}\sups{u} + \normal(0, 2 \sigma^2)
  \le - n (q_i - q_j)\bigg).
\end{equation*}
Using that $|\Delta_{ij}\sups{u}| \le 2$ and $v_{ij} \defeq \var(Y_i - Y_j)
\le q_i + q_j$ by Lemma~\ref{lemma:indicator-properties}, for $|\lambda| \le
\frac{1}{4}$ standard sub-exponential bounds~\cite[Ch.~2]{Wainwright19} give
that
$\E[\exp(\lambda \Delta_{ij})] \le \exp(\frac{3 \lambda^2 v_{ij}}{5})$
for $|\lambda| \le \frac{1}{4}$. The Chernoff bound technique then yields
\begin{align*}
  \P\bigg(\sum_{u = 1}^n \Delta^{(u)}_{ij} + \normal(0, \sigma^2) \le -t\bigg)
  & \le \exp\left(\frac{3 \lambda^2 v_{ij} n}{5}
  + \lambda^2 \sigma^2 - \lambda t\right)
  ~ \mbox{for~} |\lambda| \le \frac{1}{4}.
\end{align*}
Optimizing by setting $\lambda = \min\{\frac{1}{4},
\frac{5t}{6 n v_{ij} + 10 \sigma^2}\}$ gives
\begin{equation*}
  \P\bigg(\sum_{u = 1}^n \Delta^{(u)}_{ij} + \normal(0, \sigma^2) \le -t\bigg)
  \le \exp\left(-\min\left\{\frac{5 t^2}{12 n v_{ij} + 20 \sigma^2},
  \frac{t}{8}\right\}\right).
\end{equation*}
Substituting $t = n (q_i - q_j)$ gives the lemma.

%% file: proof-histogram.tex
\section{Proof of Proposition~\ref{proposition:element-level-histogram}}
\label{proof:element-level-mse-clustering}

For shorthand, let $\what{p}(\rho) = \frac{1}{n} \sum_{u=1}^n
\clusterproj(X\sups{u})$ be the non-privatized projection vector.
We compute the bias and bounds on the moment generating function
of each coordinate of the vector.

\paragraph{Bias of the projected vector}
We control the bias for each element cluster. Fix
$c \in \{c_1, \ldots, c_k\}$. Defining $Y = \sum_{j \in c} X_j$,
we evidently have that $X_{c} = \clusterproj(X)_{c}$ if $Y \le \rho$,
that is, the coordinate is unprojected. Thus we obtain
\begin{align}
  \lone{\E[\what{p}_{c}(\rho) - m p_{c}]}
  & = \lone{\E[(X_{c} - m p_{c}) \indic{Y \le \rho}]
    + \E[(\clusterproj(X)_{c} - m p_{c}) \indic{Y > \rho}]} \nonumber \\
  & \le \lone{\E[X_{c} - m p_{c}]}
  + \lone{\E[(\clusterproj(X)_{c} - m X_{c}) \indic{Y > \rho}]} \nonumber \\
  & \le \E\left[\lone{\clusterproj(X)_{c} - X_{c}} \indic{Y > \rho}\right].
  \label{eqn:lone-error-start}
\end{align}
To bound the remaining term in inequality~\eqref{eqn:lone-error-start}, we
present two lemmas, whose proofs we defer to
Sections~\ref{sec:proof-binbound} and~\ref{sec:proof-biasbound},
respectively.

\begin{lemma}
  \label{lemma:binbound}
  Let $X \sim \binomial(m,p), p \le 1/4$, and $l \ge 3mp$. Then
  \begin{equation*}
    \sum_{i=\lceil l \rceil}^{m} \mathbb{P}(X \ge i)
    \le 2\mathbb{P}(X \ge \lceil l \rceil)
    \le 4 \mathbb{P}(X=\lceil l \rceil).
  \end{equation*}
\end{lemma}

\begin{lemma}
  \label{lemma:biasbound}
  Let $X \sim \binomial(m,p)$, $l = 3mp + 3 \log m + t$
  and $m \ge 3$. Then
  \begin{equation*}
    \mathbb{P}(X=\lceil l \rceil) \le p 2^{-t}.
  \end{equation*}
\end{lemma}

The variable $Y \sim \binomial(m, P(c_k))$ in
expression~\eqref{eqn:lone-error-start}, and $|X_i - \clusterproj(X)_i| \le
X_i$ so $\lones{X_c - \clusterproj(X)_c} \le m$. Thus we have
\begin{align}
  \E\left[\lone{\clusterproj(X)_c - X_c} \indic{Y > l}\right]
  \le m \P(Y > l)
  & \stackrel{(i)}{\le} 4 m \P(Y = \ceil{l}) \nonumber \\
  \stackrel{(ii)}{\le} 4 m P(c) 2^{-(t + \log m)}
  \le 2^{2 - t} P(c)
  \label{eqn:histogram-bias-bound}
\end{align}
where inequality~$(i)$ is a consequence of Lemma~\ref{lemma:binbound}
and~$(ii)$ of Lemma~\ref{lemma:biasbound} once we recall
that $\rho \ge 3 mp + 4 \log m + t$.

\paragraph{Variance and moment generating function}
We have $\var(\clusterproj(X)_j)
\le \var(X_j) = m p_j(1 - p_j)$, because projections reduce variance.
We also have $\clusterproj(X)_j \le \rho$, so
as a consequence, we obtain the moment generating function bound
\begin{align*}
  \lefteqn{\E[\exp(\lambda(\clusterproj(X)_j - \E[\clusterproj(X)_j]))]} \\
  & = 1 + \frac{\lambda^2 \var(\clusterproj(X)_j)}{2}
  + \sum_{k = 3}^\infty \frac{\lambda^k \E[(\clusterproj(X)_j
      - \E[\clusterproj(X)_j])^k]}{k!} \\
  & \le 1 + \frac{\lambda^2 m p_j}{2}
  + \lambda^2 m p_j \sum_{k = 3}^\infty \frac{\lambda^{k - 2} \rho^{k-2}}{k!}
  \le \exp\left(\lambda^2 m p_j\right)
\end{align*}
for $|\lambda| \le \rho^{-1}$, where we have used that $1 + x \le e^x$.
Thus for any coordinate $j$ we have
\begin{align*}
  \E\left[\exp(\lambda n (\mechanism_j(\sample, \rho, \{c_k\})
    - \E[\mechanism_j(\sample, \rho, \{c_k\})]))\right]
  & \le \exp\left(\lambda^2 n m p_j + \frac{\lambda^2 \rho^2 \sigma^2}{2}
  \right)
\end{align*}
for all $|\lambda| \le \rho^{-1}$. Using the bias
bound~\eqref{eqn:histogram-bias-bound}, we see that
there exists a $q \in \R_+$ with $\ones^T q_c \le P(c)$ for each
cluster $c \in \{c_k\}$ such that
for any cluster $c$, $j \in c$, and $u \ge 0$ we have
\begin{align*}
  \P\left(|\mechanism_j(\sample, \rho, \{c_k\})
  - m p_j| \ge 2^{2 - t} q_j + u\right)
  \le \exp\left(\frac{\lambda^2 m p_j}{n} + \frac{\lambda^2 \sigma^2
    \rho^2}{2 n^2} - \lambda u\right).
\end{align*}

Exactly as in the proof of
Proposition~\ref{proposition:heavyhitters}
(see specifically Appendix~\ref{proof:ordererrorbound}), we thus obtain that
for a numerical constant $C > 0$ and all $u \ge 0$,
\begin{equation*}
  \P\left(\left|\mechanism_j(\sample, \rho, \{c_k\})
  - m p_j\right| \ge 2^{2 - t} q_j + u\right)
  \le \exp\left(-C \min\left\{\frac{n u^2}{m p_j},
  \frac{n^2 u^2}{\sigma^2 \rho^2},
  \frac{n u}{\rho} \right\}\right).
\end{equation*}

The variance bounds are immediate by noting that no (non-private) estimator
has mean-squared error asymptotically better than $\frac{m p_j (1 -
  p_j)}{n}$.

\subsection{Proof of Lemma~\ref{lemma:binbound}}
\label{sec:proof-binbound}

For $j \ge 3mp$ we have
\begin{equation}
  \label{eqn:ratio-binomials}
  \frac{\mathbb{P}(X=j+1)}{\mathbb{P}(X=j)} 
  = \frac{ \choose{m}{j+1} p^{j+1} (1-p)^{m-j-1} }{ \choose{m}{j} p^{j} (1-p)^{m-j} }
  =  \frac{m-j}{j} \frac{p}{1-p}
  \le \frac{4}{3} \left(\frac{m}{j} - 1\right) p
  \le \half.
\end{equation}
By quasi-convexity of the ratio function,
the last inequality implies that $\P(X \ge j + 1) / \P(X \ge j)
\le \half$ for $j \ge 3mp$.
The first inequality of the lemma now follows as $\sum_{i=0}^{\infty}
2^{-i} = 2$.  The second inequality then follows
as $\P(X \ge l) = \sum_{i = l}^\infty \P(X = i)
\le \P(X = l) \sum_{i = 0}^\infty 2^{-i}$.

\subsection{Proof of Lemma~\ref{lemma:biasbound}}
\label{sec:proof-biasbound}

Inequality~\eqref{eqn:ratio-binomials} in the proof of
Lemma~\ref{lemma:binbound} gives $\frac{\mathbb{P}(X=j+1)}{\mathbb{P}(X=j)}
\le 1/2$ for $j \ge 3mp$.  We consider two cases according to the value of
$p$.  First, if $p \le 1/m^2$, we have $mp < 1$ and so
\begin{align*}
  \P(X = \ceil{l}) & \stackrel{(i)}{\le}
  \mathbb{P}(X = \ceil{3mp + 2 \log m}) 2^{-t}
  \stackrel{(ii)}{\le} \P(X = 2) 2^{-t}
  \le m^2 p^2 2^{-t} \le p 2^{-t},
\end{align*}
where inequality~$(i)$ uses that $3 \log m > 2 \log m + 1$ and
inequality~$(ii)$ that $2 \log m > 2$, in both cases as $m \ge 3$.
In the other case, we let $p \ge 1/m^2$. Then
\begin{equation*}
  \mathbb{P}(X= \ceil{l})
  \stackrel{(iii)}{\le} \mathbb{P}(X = \ceil{3mp}) 2^{-2\log(m)-t}
  \le  \frac{2^{-t}}{m^2}
  \le p 2^{-t},
\end{equation*}
where inequality~$(iii)$ uses that $\P(X = k) \le \P(X = k - 1)$ for
$k \ge mp$ and again that $3 \log m > 2 \log m + 1$.

%% file: proof-loss-minimization.tex
\section{Proof of Theorem~\ref{theorem:normality}}
\label{sec:proof-normality}

We prove the theorem for a
more general family of updates, which include projected stochastic
gradient as a special case, known as the \textsc{aProx}
(approximate proximal point) family~\cite{AsiDu19, AsiDu19siopt};
several authors
present convergence analyses for these
methods~\cite{DuchiRu18c, DavisDr19, AsiDu19siopt, AsiDu19}.
These
methods
iteratively build a model of the loss at the current iterate and minimize the
model with
regularization. A \emph{model} of
$\loss$ at a point $\theta_0$, denoted $\loss_{\theta_0}(\cdot; \statval)$,
is a function satisfying the following conditions~\cite{DavisDr19, AsiDu19}.

\begin{enumerate}[label=(C.\roman*)]
\item \label{item:convex}
  The model $\theta \mapsto \loss_{\theta_0}(\theta; \statval)$
  is convex and subdifferentiable.
\item The model is locally accuate at $\theta_0$:
  $\loss_{\theta_0}(\theta_0; \statval) = \loss(\theta_0; \statval)$.
\end{enumerate}
\begin{enumerate}[label=(C.\roman*)]
  \setcounter{enumi}{2}
\item \label{item:lower-bound}
  The model provides a lower bound:
  $\loss_{\theta_0}(\theta; \statval) \le
  \loss(\theta; \statval)$ for $\theta \in \Theta$.
\end{enumerate}
\noindent
When the losses $\loss$ are convex and differentiable,
the simplest model is the linear (first-order) approximation
$\loss_{\theta_0}(\loss; \statval) \defeq \loss(\theta_0; \statval) +
\<\nabla \loss(\theta_0; \statval), \theta - \theta_0\>$, which
satisfies conditions~\ref{item:convex}--\ref{item:lower-bound}.

For an initial point $\theta_0$ and stepsize $\stepsize > 0$,
we define the approximate proximal
point update
\begin{equation*}
  \aprox_\stepsize(\theta_0; \statval) \defeq
  \argmin_{\theta \in \Theta}
  \left\{\loss_{\theta_0}(\theta; \statval) + \frac{1}{2 \stepsize}
  \ltwo{\theta - \theta_0}^2 \right\},
\end{equation*}
and gradient mapping
\begin{equation*}
  \gradmap_\stepsize(\theta_0; \statval)
  \defeq \frac{1}{\stepsize} \left(\theta_0 - \aprox_\stepsize(\theta_0; \statval)
  \right).
\end{equation*}
In standard application of such methods~\cite{DuchiRu18c, DavisDr19,
  AsiDu19, AsiDu19siopt}, given a datapoint
$\statval$ and stepsize $\stepsize_k$,
we update $\theta_{k + 1} = \aprox_{\stepsize_k}(\theta_k; \statval)
= \theta_k - \stepsize_k \gradmap_{\stepsize_k}(\theta_k; \statval)$.
This recovers the standard projected gradient method whenever
$\loss_\theta$ is the first-order model
$\loss_{\theta_0}(\theta; \statval)
= \loss(\theta_0; \statval)
+ \<\nabla \loss(\theta_0; \statval), \theta - \theta_0\>$
We then perform the obvious
generalization of the noisy stochastic
gradient iteration~\eqref{eqn:sgd-update-sequence}, and we will prove
the convergence guarantee claimed in the theorem for the iteration
\begin{equation}
  \label{eqn:private-update-sequence}
  \theta_{k+1}^n \defeq
  \theta_k^n - \stepsize_k \left(\gradmap_{\stepsize_k}(\theta_k^n;
  \statrv_{\randind(k)}^n) + \sigma_n
  Z_k \right).
\end{equation}


\newcommand{\goodevent}{\mc{E}}
\newcommand{\localgraderr}{\zeta}
\newcommand{\graderr}{\varepsilon}

We develop a few notational shorthands for the analysis. Let
\begin{equation*}
  \what{\theta}_n \defeq \argmin_{\theta \in \Theta}
  \risk_n(\theta)
  ~~ \mbox{and} ~~
  H_n \defeq \nabla^2 \risk_n(\what{\theta}_n)
  = \frac{1}{n} \sum_{i = 1}^n \nabla^2 \loss(\what{\theta}_n; \statrv_i).
\end{equation*}
Then under the conditions of Assumption~\ref{assumption:make-it-easy},
standard asymptotics~\cite{VanDerVaart98} give that
\begin{equation}
  \label{eqn:standard-asymptotics}
  \what{\theta}_n \cas \theta\opt ~~ \mbox{and} ~~
  \what{\theta}_n - \theta\opt = -\frac{1}{n}
  \nabla^2 \risk(\theta\opt)^{-1} \sum_{i = 1}^n
  \nabla \loss(\theta\opt; \statrv_i)
  + o_P(1/\sqrt{n}).
\end{equation}
Moreover, under Assumption~\ref{assumption:make-it-easy}, there
exists $\lambda > 0$ such if we define the event
\begin{equation*}
  \goodevent_n \defeq \left\{
  \ltwos{\what{\theta}_n - \theta\opt} \le \frac{\epsilon}{8},
  ~ \nabla^2 \risk_n(\theta) \succeq \lambda I
  ~ \mbox{for~} \theta \in \theta\opt + \epsilon \ball,
  ~ \frac{1}{n} \sum_{i = 1}^n \lipconst_a(\statrv_i)^2
  \le 2 \E[\lipconst_a(\statrv)^2]
  ~ \mbox{for~} a \in \{0,1,2\}
  \right\},
\end{equation*}
there exists a (potentially random, but finite $N$) such that
$n \ge N$ implies $\goodevent_n$ holds.

The sequence $\wb{\theta}_k^n$ defines a triangular array, which adds some
complexity to our proof and necessitates a somewhat more careful treatment,
which we now provide. Our analysis follows \citet{PolyakJu92} and
\citet{AsiDu19siopt}. We begin by defining the triply-indexed matrices
\begin{equation*}
  B_i^k(n) \defeq \stepsize_i \sum_{j = i}^k \prod_{l = i + 1}^j
  (I - \stepsize_l H_n)
  ~~ \mbox{and} ~~
  A_i^k(n) \defeq B_i^k(n) - H_n^{-1},
\end{equation*}
where we note that $H_n^{-1}$ exists on $\goodevent_n$ and satisfies
$H_n^{-1} \preceq \lambda^{-1} I$. Now, for sample size $n$, which corresponds
to running algorithm~\eqref{eqn:private-update-sequence} at the given
sample size with sample $\{\statrv_1, \ldots, \statrv_n\}$, we define
\begin{equation*}
  \loss^n_k(\theta) = \loss(\theta; \statrv_{\randind(k)}^n),
\end{equation*}
that is, the loss encountered in iteration $k$ of the algorithm with sample
size $n$, where $\randind(k)$ is the random index in $[n]$
chosen at iteration $k$. We let $\mc{F}^n = \sigma(\statrv_1, \ldots,
\statrv_n)$ denote the $\sigma$-field of the $n$ observations, and
$\mc{F}^n_k$ be the $\sigma$-field generated by $\statrv_1^n$ and the first
$k$ random indices $\randind(1), \ldots, \randind(k)$.

Now we follow \citet{AsiDu19siopt}. Let us implicitly assume the event
$\goodevent_n$ holds, so that all derivatives are defined (by
Assumption~\ref{assumption:make-it-easy}).
Define the remainder
\begin{equation*}
  R_n(\theta) \defeq \nabla \risk_n(\theta) - H_n(\theta - \what{\theta}_n)
\end{equation*}
and the localized (sub)gradient errors
\begin{equation*}
  \localgraderr_k^n \defeq \left(\nabla \loss^n_k(\theta_k^n)
  - \nabla \loss^n_k(\what{\theta}_n)\right)
  - \left(\nabla \risk_n(\theta_k^n) - \nabla \risk_n(\what{\theta}_n)\right).
\end{equation*}
Finally, we consider the \emph{model} subgradient errors, where
we note that $\theta_{k+1}^n$ satisfies
\begin{equation*}
  0 \in \partial \loss_{\theta_k^n}(\theta_{k+1}^n; \statrv_{\randind(k)}^n)
  + \frac{1}{\stepsize_k} (\theta_{k+1}^n - \theta_k^n)
  + \mc{N}_{\Theta}(\theta_{k+1}^n)
\end{equation*}
where $\mc{N}_{\Theta}(\theta) = \{v \in \R^d \mid
\<v, \tau - \theta\> \le 0, ~ \mbox{all~}\tau \in \Theta\}$ denotes the
normal cone to $\Theta$ at the point $\theta$. Thus,
there is some vector $\normalvec_{k+1}^n \in \mc{N}_{\Theta}(\theta_{k+1}^n)$
such that
\begin{equation}
  \label{eqn:normal-cone-version-update}
  0 \in \partial \loss_{\theta_k^n}(\theta_{k+1}^n; \statrv_{\randind(k)}^n)
  + \normalvec_{k+1}^n
  + \frac{1}{\stepsize_k} (\theta_{k+1}^n - \theta_k^n)
\end{equation}
where $\normalvec_{k+1}^n = 0$ if $\theta_{k+1}^n \in \interior \Theta$.
If $\theta_{k+1}^n \not\in\interior \Theta$, then
\cite[Lemma~A.5]{AsiDu19siopt} guarantees that
$\ltwo{v_{k+1}^n} \le 2 \lipobj(\statrv_{\randind(k)}^n)$ regardless.
With this, we define the \emph{model} subgradient errors
\begin{equation*}
  \graderr_k^n \defeq \nabla \loss_{\theta_k^n}(\theta_{k+1}^n; \statrv_{\randind(k)}^n)
  + \normalvec_{k+1}^n
  - \nabla \loss(\theta_k^n; \statrv_{\randind(k)}^n).
\end{equation*}
With these substitutions, we have~\cite[Eq.~(13)]{AsiDu19siopt}
that
\begin{equation*}
  \theta_{k+1}^n - \what{\theta}_n =
  (I - \stepsize_k) H_n (\theta_k^n - \what{\theta}_n)
  - \stepsize_k (\nabla \loss_k^n(\what{\theta}_n)
  + \sigma_n Z_k)
  - \stepsize_k \left(R_n(\theta_k^n) + \localgraderr_k^n
  + \graderr_k^n\right).
\end{equation*}


Then following \citet{PolyakJu92} and
\citet{AsiDu19siopt} (see Eq.~(14) of the
paper~\cite{AsiDu19siopt}, with a fixed negative sign), we have
on the event $\goodevent_n$ that
\begin{align}
  \label{eqn:proper-recursion}
  \sqrt{k} \wb{\Delta}_k^n & = \frac{1}{\sqrt{k}}
  \sum_{i = 1}^k H_n^{-1} \nabla \loss_i^n(\what{\theta}_n)
  + \sigma_n H_n^{-1} \frac{1}{\sqrt{k}}
  \sum_{i = 1}^k Z_i \\
  & \qquad + \frac{1}{\sqrt{k}}
  \sum_{i = 1}^k A_i^k(n) \left(\nabla \loss_i^n(\what{\theta}_n)
  + \sigma_n Z_i\right)
  + \frac{1}{\sqrt{k}} \sum_{i = 1}^k B_i^k(n) \left[
    R_n(\theta_i^n) + \localgraderr_i^n + \graderr_i^n
    \right]
  + O(1/\sqrt{k}),
  \nonumber
\end{align}
where the $O(1/\sqrt{k})$ term is non-random on $\goodevent_n$.  Moreover,
$\sup_{i,k,n} \indics{\goodevent_n} \opnorms{B_i^k(n)} < \infty$ as
well~\cite[Lemma 2]{AsiDu19siopt, PolyakJu92}, and independent of $n$, there
exists $\epsilon > 0$ such that for all $k \ge K(\epsilon)$,
$\indic{\goodevent_n} \frac{1}{k} \sum_{i = 1}^k \opnorms{A_i^k(n)} \le
\epsilon$.  We control each of these quantities in turn.

\begin{lemma}
  \label{lemma:convergence-of-iterates}
  Define $\delta_{k,n} \defeq
  \ltwos{\theta_k^n - \what{\theta}_n} \indic{\goodevent_n}$.
  Then
  \begin{equation}
    \label{eqn:one-step-recursion-ugh}
    \E[\delta_{k+1,n}^2 \mid \mc{F}^n] \le
    (1 - c_0 \stepsize_k) \delta_{k,n}^2
    + \stepsize_k^2 \E[\lipobj(\statrv)^2]
  \end{equation}
  for a constant $c_0 > 0$ that depends only on $\lambda > 0$ in the
  definition of $\goodevent_n$ and $\Theta$.  Additionally,
  for some $C < \infty$ independent of $n$ and $k$, we have
  \begin{equation*}
    \E[\delta_{k,n}^2 \mid \mc{F}^n] \le C \stepsize_k \log k.
  \end{equation*}
\end{lemma}

\begin{lemma}
  \label{lemma:remainders-go-to-zero}
  Let Assumption~\ref{assumption:make-it-easy} hold.
  If $k(n) \to \infty$ as $n \to \infty$, then
  $\frac{1}{\sqrt{k(n)}} \sum_{i = 1}^{k(n)} \norm{R_n(\theta_i^n)} \cp 0$.
\end{lemma}

\begin{lemma}
  \label{lemma:local-grad-errors-go-to-zero}
  Let Assumption~\ref{assumption:make-it-easy} hold. If
  $k(n) \to \infty$ as $n \to \infty$, then
  $\frac{1}{\sqrt{k(n)}} \sum_{i = 1}^{k(n)} B_i^k(n) \localgraderr_i^n \cp 0$.
\end{lemma}

\begin{lemma}
  \label{lemma:subgrad-errors-go-to-zero}
  Let Assumption~\ref{assumption:make-it-easy} hold.
  If $k(n) \to \infty$ as $n \to \infty$, then
  $\frac{1}{\sqrt{k(n)}} \sum_{i = 1}^{k(n)} \ltwo{\graderr_i^n} \cp 0$.
\end{lemma}

\begin{lemma}
  \label{lemma:matrix-inverse-errors-go-to-zero}
  Let Assumption~\ref{assumption:make-it-easy} hold and
  $\lim_n \sigma_n = \sigma \in \openright{0}{\infty}$.
  If $k(n) \to \infty$ as $n \to \infty$, then
  \begin{equation*}
    \frac{1}{\sqrt{k(n)}} \sum_{i = 1}^{k(n)} A_i^k(n)
    (\nabla \loss_i^n(\what{\theta}_n) + \sigma_n Z_i) \cp 0.
  \end{equation*}
\end{lemma}

\noindent
We prove the lemmas in Appendices~\ref{sec:proof-convergence-of-iterates},
\ref{sec:proof-remainders-go-to-zero},
\ref{sec:proof-local-grad-errors-go-to-zero},
\ref{sec:proof-subgrad-errors-go-to-zero},
\ref{sec:proof-matrix-inverse-errors-go-to-zero},
respectively.

\newcommand{\countvar}{M}

Combining the preceding three lemmas into the
recursion~\eqref{eqn:proper-recursion}, we see that if
$k = k(n) \to \infty$ as $n \to \infty$, we use that
$\sup_{i,k,n} \indics{\goodevent_n} \opnorms{B_i^k(n)} < \infty$
and that $\goodevent_n$ occurs eventually with probability 1 to write
\begin{equation}
  \label{eqn:proper-recursion-redux}
  \sqrt{k} \wb{\Delta}_k^n =
  \underbrace{
    \frac{1}{\sqrt{k}} \sum_{i = 1}^k H_n^{-1} \nabla \loss_i^n(\what{\theta}_n)
  }_{\eqdef \mc{T}_{1,n}}
  + \underbrace{\frac{1}{\sqrt{k}} \sum_{i = 1}^k H_n^{-1} \sigma_n Z_i}_{
    \eqdef \mc{T}_{2,n}}
  + o_P(1),
\end{equation}
where the $o_P(1)$ term converges to 0 in probability as $n \uparrow
\infty$.
From this point in the proof, we will treat $k$ as a function of $n$
implicitly, noting that $k = k(n)$ satisfies $\lim_n k(n) / n = \gamma$.
The recursion~\eqref{eqn:proper-recursion-redux} takes a form
similar to a multiplier central limit theorem~\cite{VanDerVaartWe96},
allowing us to precisely compute its asymptotics by computing the
asymptotics of $\mc{T}_{1,n}$ and $\mc{T}_{2,n}$, which are (asymptotically)
independent. Let
$\countvar_{n,i} \in \N$ denote the number of times observation
$\statrv_i$ is chosen in the sampling procedure to
generate $\theta_k^n$ after $k = k(n)$ iterations,
noting that $(\countvar_{n,i})_{i=1}^n
\sim \multinomial(k, \ones / n)$ is multinomial-distributed with
probabilities $1/n$, and $\sum_{i=1}^n \countvar_{n,i} = k$.
Thus we have
\begin{equation*}
  \mc{T}_{1,n}
  = \frac{1}{\sqrt{k}} \sum_{i = 1}^n
  H_n^{-1} \countvar_{n,i} \nabla \loss(\what{\theta}_n; \statrv_i).
\end{equation*}
On the event $\goodevent_n$,
a Taylor expansion yields
\begin{equation*}
  \nabla \loss(\what{\theta}_n; \statrv_i)
  = \nabla \loss(\theta\opt; \statrv_i)
  + \left(\nabla^2 \loss(\theta\opt; \statrv_i)
  + E_{n,i}\right) (\what{\theta}_n - \theta\opt),
\end{equation*}
where $\opnorms{E_{n,i}} \le \liphess(\statrv_i) \ltwos{\what{\theta}_n
  - \theta\opt}$ by Assumption~\ref{assumption:make-it-easy}. Rearranging
the count-based recursion thus gives
\begin{align}
  \mc{T}_{1,n}
  & =
  \frac{1}{\sqrt{k}}
  H_n^{-1} \sum_{i = 1}^n \countvar_{n,i} \nabla \loss(\theta\opt; \statrv_i)
  + \frac{1}{\sqrt{k}}
  H_n^{-1} \bigg(\sum_{i = 1}^n \countvar_{n,i}
  (\nabla^2 \loss(\theta\opt; \statrv_i) + E_{n,i}) \bigg)
  (\what{\theta}_n - \theta\opt)
  \nonumber \\
  & = \frac{1}{\sqrt{k}}
  H_n^{-1} \sum_{i = 1}^n \countvar_{n,i} \nabla \loss(\theta\opt; \statrv_i)
  \label{eqn:now-we-get-somewhere}
  \\
  & \qquad ~ - \frac{1}{\sqrt{k}}
  H_n^{-1} \bigg(\frac{1}{n} \sum_{i = 1}^n \countvar_{n,i}
  (\nabla^2 \loss(\theta\opt; \statrv_i) + E_{n,i}) \bigg)
  \bigg(\nabla^2 \risk(\theta\opt)^{-1}
  \sum_{i = 1}^n \nabla \loss(\theta\opt; \statrv_i)
  + o_P(\sqrt{n})\bigg).
  \nonumber
\end{align}
Now, we use that $\E[\countvar_{n,i}] = k/n$ and
$\var(\countvar_{n,i}) = k/n(1 - 1/n)$, with $\cov(\countvar_{n,i},
\countvar_{n,j}) = -k/n^2$, independently of $X_i$, to obtain
$\frac{1}{n} \sum_{i = 1}^n \countvar_{n,i}(\nabla^2 \loss(\theta\opt; \statrv_i)
+ E_{n,i}) = \frac{k}{n} \frac{1}{n} \sum_{i=1}^n \nabla^2 \loss(\theta\opt;
\statrv_i) + o_P(1)
= \frac{k}{n} H_n + o_P(1)$,
so that expansion~\eqref{eqn:now-we-get-somewhere} becomes
\begin{align}
  \mc{T}_{1,n}
  & = \frac{1}{\sqrt{k}} H_n^{-1} \sum_{i = 1}^n \countvar_{n,i}
  \nabla \loss(\theta\opt; \statrv_i)
  - \frac{1}{\sqrt{k}} \frac{k}{n} (1 + o_P(1)) \nabla^2 \risk(\theta\opt)^{-1}
  \sum_{i = 1}^n \nabla \loss(\theta\opt; \statrv_i) + o_P(1) \nonumber \\
  & = \frac{1}{\sqrt{k}}
  (\nabla^2 \risk(\theta\opt) + o_P(1))^{-1}
  \sum_{i = 1}^n \left(\countvar_{n,i} - \frac{k}{n}\right)
  \nabla \loss(\theta\opt; \statrv_i)
  + \sqrt{\frac{k}{n}} o_P(1)
  \label{eqn:can-i-apply-multiplier-clt}
\end{align}
where the error $o_P(1) \cp 0$ as $n \uparrow \infty$.

Substituting expression~\eqref{eqn:can-i-apply-multiplier-clt} into
the expansion~\eqref{eqn:proper-recursion-redux} and renormalizing
by $\sqrt{n}$ instead of $\sqrt{k}$,
\begin{equation*}
  \sqrt{n} \wb{\Delta}_k^n
  = (\nabla^2 \risk(\theta\opt) + o_P(1))^{-1}
  \left[\sqrt{\frac{n}{k}} \cdot \frac{1}{\sqrt{n}} \sum_{i = 1}^n
    \sqrt{\frac{n}{k}} \left(\countvar_{n,i} - \frac{k}{n}\right)
    \nabla \loss(\theta\opt; \statrv_i)
    + \sqrt{\frac{n}{k}} \frac{1}{\sqrt{k}} \sigma_n \sum_{i = 1}^k Z_i
    \right] + o_P(1).
\end{equation*}
Now, note that by the classic multiplier central limit theorems
(cf.~\cite[Chapters 2.9 and 3.6]{VanDerVaartWe96}),
using that $(n/k) \var(\countvar_{n,i}) = 1 - 1/n$
we have the joint convergence
\begin{equation*}
  \left(\frac{1}{\sqrt{n}} \sum_{i = 1}^n
  \sqrt{\frac{k}{n}} \left(\countvar_{n,i} - \frac{k}{n}\right)
  \nabla \loss(\theta\opt; \statrv_i),
  \frac{1}{\sqrt{n}} \sum_{i = 1}^n \nabla \loss(\theta\opt; \statrv_i)
  \right)
  \cd \normal\left(\left[\begin{matrix} 0 \\ 0 \end{matrix} \right],
  \left[\begin{matrix} \losscov & 0 \\ 0 & \losscov \end{matrix}\right]
  \right).
\end{equation*}
Adding and subtracting $\sqrt{n} (\what{\theta}_n - \theta\opt)$ as
in the standard asymptotic expansion~\eqref{eqn:standard-asymptotics},
we have
\begin{align*}
  \sqrt{n}(\wb{\theta}_k^n - \theta\opt)
  & = \nabla^2 \risk(\theta\opt)^{-1}
  \bigg[\sqrt{\frac{n}{k}}
    \frac{1}{\sqrt{n}} \sum_{i = 1}^n
    \sqrt{\frac{k}{n}} \left(\countvar_{n,i} - \frac{k}{n}\right)
    \nabla \loss(\theta\opt; \statrv_i)
    \ldots \\
    & \qquad\qquad\qquad\qquad ~
    + \frac{1}{\sqrt{n}} \sum_{i = 1}^n \nabla \loss(\theta\opt; \statrv_i)
    + \sqrt{\frac{n}{k}}
    \frac{\sigma_n}{\sqrt{k}} \sum_{i = 1}^k Z_i
    \bigg]
  + o_P(1) \\
  & \cd \normal\left(0,
  \nabla^2 \risk(\theta\opt)^{-1}
  \left((1 + 1 / \gamma) \losscov
  + (1/\gamma) \Zcov \sigma^2 \right)
  \nabla^2 \risk(\theta\opt)^{-1} \right).
\end{align*}


\subsection{Proof of Lemma~\ref{lemma:convergence-of-iterates}}
\label{sec:proof-convergence-of-iterates}

We have~\cite[Lemma~3.4]{AsiDu19siopt} that
\begin{align*}
  \half \ltwos{\theta_{k+1}^n - \what{\theta}_n}^2
  & \le \half \ltwos{\theta_k^n - \what{\theta}_n}^2
  - \stepsize_k \left[\loss_k^n(\theta_k^n) - \loss_k^n(\what{\theta}_n)
    \right]
  + \frac{\stepsize_k^2}{2} \ltwo{\nabla\loss_k^n(x_k^n)}^2 \\
  & \le \half \ltwos{\theta_k^n - \what{\theta}_n}^2
  - \stepsize_k \left[\loss_k^n(\theta_k^n) - \loss_k^n(\what{\theta}_n)
    \right]
  + \frac{\stepsize_k^2}{2} \lipobj(\statrv_{\randind(k)}^n)^2.
\end{align*}
Taking expectations conditional on $\mc{F}_{k-1}^n$, the $\sigma$-field
of the sample $\{\statrv_i\}_{i=1}^n$ and the first $k-1$ random indices
$\randind(1), \ldots, \randind(k-1)$, and noting that
$\goodevent_n \in \mc{F}^n \subset \mc{F}_{k-1}^n$, we have
\begin{align}
  \nonumber \half \indic{\goodevent_n}
  \E[\ltwos{\theta_{k+1}^n - \what{\theta}_n}^2 \mid \mc{F}_{k-1}^n]
  & \le \indic{\goodevent_n}
  \cdot \left\{
  \half \ltwos{\theta_k^n - \what{\theta}_n}^2
  - \stepsize_k \left[\risk_n(\theta_k^n) - \risk_n(\what{\theta}_n)
    \right]
  + \sum_{i=1}^n \frac{\stepsize_k^2}{2n} \lipobj(\statrv_i)^2
  \right\} \\
  & \le \indic{\goodevent_n} \cdot \left\{
  \frac{1 - c_0 \stepsize_k}{2} \ltwos{\theta_k^n
    - \what{\theta}_n}^2
  + \stepsize_k^2 \E[\lipobj(\statrv)^2] \right\},
  \nonumber
\end{align}
where $c_0 > 0$ is a constant depending on $\lambda, \Theta$, which is
positive because $\nabla^2 \risk_n(\theta) \succeq \lambda I$ for $\theta$
near $\theta \opt$.  In particular, with the definition $\delta_{k,n}
\defeq \ltwos{\theta_k^n - \what{\theta}_n} \indic{\goodevent_n}$, then
integrating over the indices $\randind(k)$ gives the
result~\eqref{eqn:one-step-recursion-ugh}.

The second result follows exactly as in the proof of Lemma~A.2 of the
paper~\cite{AsiDu19siopt} (see specifically inequality~(17) in the
\texttt{arXiv} technical report version).

\subsection{Proof of Lemma~\ref{lemma:remainders-go-to-zero}}
\label{sec:proof-remainders-go-to-zero}

On the event $\goodevent_n$, $\risk_n$ has $\sqrt{2
  \E[\lipgrad(\statrv)^2]}$-Lipschitz gradient on $\Theta$, and so a Taylor
approximation gives that for some $C < \infty$ independent of $n$ and $k$,
$R_n(\theta) \le C \ltwos{\theta - \what{\theta}_n}^2$. Thus
\begin{equation*}
  \E[\ltwo{R_n(\theta_k^n)} \indic{\goodevent_n}]
  \le C \E[\indic{\goodevent_n} \ltwos{\theta_k^n - \what{\theta}_n}^2]
  \le C \stepsize_k \log k,
\end{equation*}
where we have used Lemma~\ref{lemma:convergence-of-iterates}. Thus
\begin{equation*}
  \frac{1}{\sqrt{k}} \sum_{i = 1}^k \E[\ltwo{R_n(\theta_k^n)}
    \indic{\goodevent_n}]
  \le \frac{C \log k}{\sqrt{k}} \sum_{i = 1}^k \stepsize_i
  \le C k^{1 - \steppow - \half} \log k
  \to 0
\end{equation*}
as $k \uparrow \infty$. As $\goodevent_n$ happens eventually,
we have the result.

\subsection{Proof of Lemma~\ref{lemma:local-grad-errors-go-to-zero}}
\label{sec:proof-local-grad-errors-go-to-zero}

Fixing the sample $\{\statrv_i\}_{i=1}^n$, the localized subgradient
errors $\localgraderr_k^n$ are a martingale sequence adapted
to $\mc{F}_k^n = \sigma(\mc{F}^n, \randind(1), \ldots, \randind(k))$, the $\sigma$-field
of $\mc{F}^n$ and the random indices of the iteration through time $k$.
Moreover, $\goodevent_n \in \mc{F}^n$ and $B_i^k(n) \in \mc{F}^n$ for
all $i, k$.
Thus
\begin{equation*}
  \E\left[\ltwobigg{\sum_{i = 1}^k
      B_i^k(n) \localgraderr_i^n}^2 \mid \mc{F}^n\right]
  = \sum_{i = 1}^k \E\left[\ltwos{B_i^k(n)\localgraderr_i^n}^2
    \mid \mc{F}^n\right].
\end{equation*}
Now, we note that if $\theta_k^n, \what{\theta}_n \in \theta\opt + \epsilon
\ball$, then
\begin{equation*}
  \ltwo{\localgraderr_k^n}
  \le \left(\lipgrad(\statrv_{\randind(k)}^n)
  + \frac{1}{n} \sum_{i = 1}^n \lipgrad(\statrv_i)\right)
  \ltwo{\theta_k^n - \what{\theta}_n},
\end{equation*}
while otherwise we have
\begin{equation*}
  \ltwo{\localgraderr_k^n}
  \le 2 \lipobj(\statrv_{\randind(k)}^n)
  + \frac{2}{n} \sum_{i = 1}^n \lipobj(\statrv_i).
\end{equation*}
In either case, on the event $\goodevent_n$, the compactness of
$\Theta$ guarantees that there exists some $C < \infty$
independent of $n$ and $k$ such that
\begin{equation*}
  \indic{\goodevent_n}
  \ltwo{\localgraderr_k^n}
  \le C \cdot \indic{\goodevent_n}
  \left(\lipobj(\statrv_{\randind(k)}^n)
  + \lipgrad(\statrv_{\randind(k)}^n) + \E[\lipgrad(\statrv)^2]^{1/2}\right)
  \ltwos{\theta_k^n - \what{\theta}_n}.
\end{equation*}
In particular, as $\theta_k^n \in \mc{F}_{k-1}^n$, we obtain that
\begin{equation*}
  \indic{\goodevent_n}
  \E\left[\ltwo{\localgraderr_k^n}^2 \mid \mc{F}_{k-1}^n\right]
  \le C \cdot \indic{\goodevent_n}
  \sqrt{\E[\lipobj(\statrv)^2] + \E[\lipgrad(\statrv)^2]}
  \cdot \ltwobig{\theta_k^n - \what{\theta}_n}^2.
\end{equation*}
As $\sup_{i,k,n} \indic{\goodevent_n} \opnorm{B_i^k(n)} < \infty$,
we have
\begin{align*}
  \indic{\goodevent_n}
  \E\left[\ltwobigg{\sum_{i=1}^k B_i^k(n) \localgraderr_i^n}^2
    \mid \mc{F}^n\right]
  & \le C \indic{\goodevent_n}
  \sqrt{\E[\lipobj(\statrv)^2 + \lipgrad(\statrv)^2]}
  \sum_{i = 1}^k \E\left[\ltwobig{\theta_i^n - \what{\theta}_n}^2
    \mid \mc{F}^n\right] \\
  & \le C \log k \sum_{i = 1}^k \stepsize_i,
\end{align*}
where the final inequality uses Lemma~\ref{lemma:convergence-of-iterates}.
Dividing by $k$ gives the result.

\subsection{Proof of Lemma~\ref{lemma:subgrad-errors-go-to-zero}}
\label{sec:proof-subgrad-errors-go-to-zero}

We continue to build off of \citet{AsiDu19siopt}. By Lemma~A.4 (a
specialization of \cite[Thm.~6.1]{DavisDrPa17}) of their paper, as
$\loss(\cdot; \statval)$ has $\lipgrad(\statval)$-Lipschitz gradient on
$\theta\opt + \epsilon \ball \subset \interior \Theta$, we have (see also
\cite[Eq.~(15)]{AsiDu19siopt}) that whenever $\theta_k^n, \theta_{k+1}^n \in
\theta\opt + (\epsilon/4) \ball$,
\begin{equation*}
  \ltwo{\graderr_k^n}
  \le 2 \lipgrad(\statrv_{\randind(k)}^n) \ltwos{\theta_k^n - \theta_{k+1}^n}
  \le \stepsize_k \lipgrad(\statrv_{\randind(k)}^n)^2
  + \stepsize_k \ltwo{\nabla \loss_k^n(\theta_k^n)}^2.
\end{equation*}
We also always have $\ltwo{\graderr_k^n} \le 4 \lipobj(\statrv_{\randind(k)}^n)$ by
the triangle inequality applied to the
containment~\eqref{eqn:normal-cone-version-update}.
Consequently, we obtain that
\begin{align*}
  \frac{1}{\sqrt{k}} \sum_{i = 1}^k
  \ltwo{\graderr_i^n}
  & \le \frac{4}{\sqrt{k}}
  \sum_{i = 1}^k \indic{\ltwo{\theta_i^n - \theta\opt} \ge \epsilon/4,
    \ltwo{\theta_{i+1}^n - \theta\opt} \ge \epsilon/4}
  \lipobj(\statrv_{\randind(i)}^n) \\
  & \qquad ~ + 
  \frac{1}{\sqrt{k}}
  \sum_{i = 1}^k \stepsize_i \left(\lipgrad(\statrv_{\randind(i)}^n)^2
  + \ltwo{\nabla \loss_i^n(\theta_i^n)}^2\right).
\end{align*}
Now, we use the triangle inequality to see that on the event
$\goodevent_n$, as $\ltwos{\what{\theta}_n - \theta\opt} \le \epsilon/8$,
to have $\ltwo{\theta_{i+1}^n - \theta\opt} \ge \epsilon/4$ we must have
$\ltwo{\theta_{i+1}^n - \what{\theta}_n} \ge \epsilon/8$. Moreover, for
this to be the case, the Lipschitz continuity of $\loss$ over
$\Theta$ and that $\ltwo{\theta_{i+1}^n - \theta_i^n} \le
\stepsize_i \lipobj(\statrv_{\randind(i)}^n)$ together give that
\begin{equation*}
  \ltwo{\theta_{i+1}^n - \theta\opt} \ge \frac{\epsilon}{4}
  ~~ \mbox{implies} ~~
  \ltwo{\theta_{i}^n - \what{\theta}_n}
  + \stepsize_{i} \lipobj(\statrv_{\randind(i)}^n)
  \ge \frac{\epsilon}{8}.
\end{equation*}
Thus, revisiting the previous display, we have on $\goodevent_n$ that
\begin{align*}
  \frac{1}{\sqrt{k}} \sum_{i = 1}^k
  \ltwo{\graderr_i^n}
  & \le \frac{4}{\sqrt{k}}
  \sum_{i = 1}^k \left(2 \cdot
  \indic{\ltwos{\theta_i^n - \what{\theta}_n} \ge \epsilon/16}
  + \indic{\stepsize_i \lipobj(\statrv_{\randind(i)}^n)
    \ge \epsilon / 8}\right)
  \lipobj(\statrv_{\randind(i)}^n) \\
  & \qquad ~ + 
  \frac{1}{\sqrt{k}}
  \sum_{i = 1}^k \stepsize_i \left(\lipgrad(\statrv_{\randind(i)}^n)^2
  + \ltwo{\nabla \loss_i^n(\theta_i^n)}^2\right).
\end{align*}
Taking expectations conditional
on $\mc{F}^n$ and using that $\goodevent_n \in \mc{F}^n$, we have
on the event $\goodevent_n$ that
\begin{align}
  \lefteqn{\E\bigg[\frac{1}{\sqrt{k}} \sum_{i = 1}^k
      \ltwo{\graderr_i^n} \mid \mc{F}^n\bigg]} \nonumber \\
  & \le
  \frac{8}{\sqrt{k}} \sum_{i = 1}^k \E\left[\lipobj(\statrv_{\randind(i)}^n)
    \indic{\ltwos{\theta_i^n - \what{\theta}_n} \ge \epsilon/16}
    \mid \mc{F}^n\right]
  + \frac{4}{\sqrt{k}}
  \sum_{i = 1}^k \frac{1}{n} \sum_{j = 1}^n \lipobj(\statrv_j)
  \indic{\stepsize_i \lipobj(\statrv_j) \ge \epsilon / 8} \nonumber \\
  & \qquad ~ + \frac{4}{\sqrt{k}}
  \sum_{i = 1}^k \stepsize_i \left(\E[\lipgrad(\statrv)^2]
  + \E[\lipobj(\statrv)^2]\right),
  \label{eqn:clif-bars-are-lunch}
\end{align}
where we have used that on $\goodevent_n$, $\frac{1}{n}
\sum_{i = 1}^n \lipconst_a(\statrv_i)^2 \le 2 \E[\lipconst_a(\statrv)^2]$
for $a \in \{0, 1, 2\}$.

We now control the first terms in the
righthand sum of inequality~\eqref{eqn:clif-bars-are-lunch}.
For the second, we note that if $Y$ is a random variable
with $\E[Y^2] \le C$, then
\begin{equation*}
  \E[Y \indic{\stepsize Y \ge \epsilon}]
  \le \sqrt{\E[Y^2] \P(\stepsize Y \ge \epsilon)}
  \le \sqrt{\stepsize^2 \E[Y^2] \E[Y^2] / \epsilon^2}
  \le C \stepsize / \epsilon
\end{equation*}
by the Cauchy-Schwarz and Chebyshev inequalities, so that on
event $\goodevent_n$ that $\frac{1}{n} \sum_{i = 1}^n \lipobj(\statrv_i)^2
\le 2 \E[\lipobj(\statrv)^2]$, we have
\begin{equation*}
  \frac{1}{n} \sum_{j = 1}^n \lipobj(\statrv_j) \indic{
    \stepsize_i \lipobj(\statrv_j) \ge \epsilon / 8}
  \le \frac{16}{\epsilon} \stepsize_i.
\end{equation*}
For the first term in the right side of~\eqref{eqn:clif-bars-are-lunch},
recalling the definition $\delta_{k,n} = \ltwos{\theta_k^n -
  \what{\theta}_n} \indics{\goodevent_n}$ in
Lemma~\ref{lemma:convergence-of-iterates}, we use that $\theta_i^n \in
\mc{F}^n_{i-1}$ to obtain
\begin{align*}
  \indic{\goodevent_n}
  \E\left[\lipobj(\statrv_{\randind(i)}^n) \indic{\ltwos{\what{\theta}_n
        - \theta_i^n} \ge \epsilon / 16} \mid \mc{F}^n\right]
  & = \frac{1}{n} \sum_{j = 1}^n \lipobj(\statrv_j)
  \P\left(\delta_{i,n} \ge \epsilon / 16 \mid \mc{F}^n\right) \\
  & \le \frac{1}{n} \sum_{j = 1}^n \lipobj(\statrv_j)
  \frac{C \stepsize_i \log i}{\epsilon^2},
\end{align*}
where the inequality is a consequence of
Lemma~\ref{lemma:convergence-of-iterates} and Chebyshev's inequality.
Returning to inequality~\eqref{eqn:clif-bars-are-lunch},
we find that
\begin{equation*}
  \indic{\goodevent_n}
  \E\bigg[\frac{1}{\sqrt{k}} \sum_{i = 1}^k \ltwo{\graderr_i^n} \mid \mc{F}^n
    \bigg]
  \le \frac{C}{\sqrt{k}}
  \sum_{i = 1}^k \stepsize_i \log i
\end{equation*}
where $C < \infty$ may depend on problem parameters (e.g.\ $\epsilon$ and
$\E[\lipconst_a(X)^2]$) but is independent of $k$ and $n$.
As $\sum_{i = 1}^k \stepsize_i \log i / \sqrt{k}
= O(1) k^{1 - \steppow - 1/2 + \epsilon}$ for any $\epsilon > 0$, and
$\goodevent_n$ occurs with probability one eventually,
taking expectations over $\mc{F}^n$ gives the lemma.


\subsection{Proof of Lemma~\ref{lemma:matrix-inverse-errors-go-to-zero}}
\label{sec:proof-matrix-inverse-errors-go-to-zero}

Recall that on $\goodevent_n$, if $k = k(n) \to \infty$
then $\frac{1}{k} \sum_{i = 1}^k \opnorms{A_i^k(n)} \to 0$. As conditional
on $\mc{F}^n$ we have $\E[\nabla \loss_i^n(\what{\theta}_n) \mid \mc{F}^n] =
\frac{1}{n} \sum_{i = 1}^n \nabla \loss(\what{\theta}_n; \statrv_i) = 0$ on
$\goodevent_n$, and the $Z_i$ are mean-zero independent of $\mc{F}^n$ with
$\cov(Z_i) = \Zcov$, we have
\begin{align*}
  \lefteqn{\indic{\goodevent_n}
    \E\left[
      \ltwobigg{\frac{1}{\sqrt{k}} \sum_{i = 1}^k A_i^k(n)
        (\nabla \loss_i^n(\what{\theta}_n)
        + \sigma_n Z_i)}^2 \mid \mc{F}^n\right]} \\
  & = \indic{\goodevent_n} \frac{1}{k} \sum_{i = 1}^k
  \frac{1}{n} \sum_{j = 1}^n \ltwo{A_i^k(n)
    \nabla \loss(\what{\theta}_n; \statrv_j)}^2
  + \indic{\goodevent_n}
  \frac{\sigma_n^2}{k} \sum_{i = 1}^k \tr(A_i^k(n) \Zcov A_i^k(n)) \\
  & \le \indic{\goodevent_n}
  \frac{2\E[\lipobj(\statrv)^2]}{k} \sum_{i = 1}^k \opnorm{A_i^k(n)}^2
  \le \indic{\goodevent_n}
  \frac{\sigma_n^2}{k} \tr(\Zcov) \sum_{i = 1}^k \opnorm{A_i^k(n)}^2
  \to 0
\end{align*}
as $k \to \infty$, because $\sup_{i,k,n} \opnorms{A_i^k(n)}
\indics{\goodevent_n} < \infty$. That $\goodevent_n$ occurs
eventually gives the lemma.